\documentclass[12pt]{article}

\usepackage[margin=1in]{geometry}
\usepackage{xcolor}				 
\usepackage{graphicx}            
\usepackage{float}               

\usepackage{titlesec}            
\usepackage{amsmath}             
\usepackage{amssymb}             
\usepackage{amsthm}              

\usepackage{mathtools}		     
\usepackage{bm} 				 
\usepackage{dsfont}              

\usepackage{physics}             
\usepackage{algorithmic}
\usepackage{soul}
\usepackage{textcomp}
\usepackage{multirow}            
\usepackage{booktabs}
\usepackage{makecell}
\usepackage{subcaption}
\usepackage{enumerate}           
\usepackage{bookmark}            
\usepackage{float}

\restylefloat{table}

\usepackage[
    style=authoryear,
    natbib=true,
    maxcitenames=2
]{biblatex}
\usepackage[
    vlined,
    ruled,
    linesnumbered,
]{algorithm2e}

\SetCommentSty{mycommfont}

\usepackage{hyperref}

\usepackage[
    noabbrev,
    capitalise,
    nameinlink
]{cleveref}

\hypersetup{
    hidelinks=true
}

\newcommand{\E}{\ensuremath{\mathbb{E}}}    
\newcommand{\Tau}{\mathcal{T}}              
\newcommand{\Oh}{\mathcal{O}}               

\newtheorem{proposition}{Proposition}
\crefname{proposition}{Proposition}{Propositions}

\crefname{assume}{Assumption}{Assumptions}

\crefname{lemma}{Lemma}{Lemmas}

\newtheorem{thm}{Theorem}
\crefname{thm}{Theorem}{Theorems}

\DeclarePairedDelimiterX{\divergex}[2]{(}{)}{%
    #1\;\delimsize\|\;#2%
}
\newcommand{\kld}{D_\text{KL}\divergex}

\DeclareMathOperator*{\argmax}{arg\,max}    
\DeclareMathOperator{\zum}{sum}             

\newcommand{\MDP}{\ensuremath{\mathcal{M}}}
\newcommand{\States}{\ensuremath{\mathcal{S}}}
\newcommand{\Actions}{\ensuremath{\mathcal{A}}}
\newcommand{\Parent}{\ensuremath{\mathcal{P}}}
\newcommand{\Children}{\ensuremath{\mathcal{C}}}
\newcommand{\Data}{\ensuremath{\mathcal{D}}}

\newcommand{\Ie}{\textit{I.e.}}             
\newcommand{\ie}{\textit{i.e.}}             
\newcommand{\eg}{\textit{e.g.}}             
\newcommand{\Eg}{\textit{E.g.}}             
\newcommand{\nb}{\textit{n.b.}}             
\newcommand{\Nb}{\textit{N.b.}}             
\newcommand{\cf}{\textit{cf}.}              
\newcommand{\etal}{\textit{et al.}}         
\newcommand{\naive}{na\"\i{}ve}             
\newcommand{\vav}{vis-\'a-vis}               

\newcommand{\wrt}{\text{w.r.t.}}            

\let\oldnl\nl
\newcommand{\nonl}{\renewcommand{\nl}{\let\nl\oldnl}}

\titleformat*{\paragraph}{\bfseries}

\begin{document}

\title{%
    Revisiting Maximum Entropy Inverse Reinforcement Learning:
    \\
    New Perspectives and Algorithms%
}

\author{%
    Aaron J. Snoswell\thanks{School of Information Technology and Electrical Engineering, The University of Queensland}
    \and
    Surya P. N. Singh\thanks{Intuitive Surgical}
    \and
    Nan Ye\thanks{School of Mathematics and Physics, The University of Queensland}
}

\newcommand{\note}[1]{Note: #1}

\maketitle

\begin{abstract}
    We provide new perspectives and inference algorithms for Maximum Entropy (MaxEnt) Inverse Reinforcement Learning (IRL), which provides a principled method to find a most non-committal reward function consistent with given expert demonstrations, among many consistent reward functions.
	
	We first present a generalized MaxEnt formulation based on minimizing a KL-divergence instead of maximizing an entropy.
	This improves the previous heuristic derivation of the MaxEnt IRL model (for stochastic MDPs), allows a unified view of MaxEnt IRL and Relative Entropy IRL, and leads to a model-free learning algorithm for the MaxEnt IRL model.
	Second, a careful review of existing inference algorithms and implementations showed that they approximately compute the marginals required for learning the model.
	We provide examples to illustrate this, and present an efficient and exact inference algorithm.
	Our algorithm can handle variable length demonstrations; in addition, while a basic version takes time quadratic in the maximum demonstration length $L$, an 
	improved version of this algorithm reduces this to linear using a padding trick.
	
	Experiments show that our exact algorithm improves reward learning as compared to the approximate ones.
	Furthermore, our algorithm scales up to a large, real-world dataset involving driver behaviour forecasting.
	We provide an optimized implementation at our open-source project repository.
	Our new insight and algorithms could possibly lead to further interest and exploration of the original MaxEnt IRL model.
\end{abstract}

\section{Introduction}

Inverse Reinforcement Learning (IRL) searches for a reward or cost
function to rationalize observed behaviour.
This is challenging because the same reward may be optimized by different
behaviors, and optimizing different reward functions can lead to the same
behavior.
In their seminal work Ziebart \etal{} developed a principled solution using the Maximum Entropy
(MaxEnt) principle to choose the most non-committal consistent reward \cite{Ziebart2008} -- \ie{} a reward which matches demonstrated feature counts but makes no additional assumptions about the demonstrated behaviour.
Variations of this idea have seen great success in many recent works --- including models based on causal entropy \cite{Ziebart2010a, Ziebart2010b}, and efficient sample-based methods that maximize state-conditioned policy entropy \cite{Finn2016a,Fu2017}.

Despite the long line of works based on the original MaxEnt IRL paper, we
believe that the full value of the original MaxEnt IRL model might not have been fully realized yet, for two reasons.
First, while the MaxEnt IRL model for deterministic Markov Decision Processes
(MDPs) has been rigorously derived from the MaxEnt principle, the corresponding model for
stochastic MDPs was based on a heuristic argument.
While this has motivated alternative formulations such as causal Maximum Causal Entropy (MaxCausalEnt) IRL
\cite{Ziebart2010a}, we were also interested to know whether there is a simple
and rigorous justification of the MaxEnt IRL model for stochastic MDPs, which might bring new insight.
Second, the published inference (\ie{}, marginal computation) algorithms
\cite{Ziebart2008,Ziebart2010} for MaxEnt IRL only approximately compute the
marginals required for learning -- as confirmed by the original
authors.\footnote{Personal correspondence with B. Ziebart.}
Existing implementations online have largely followed these algorithms and are
approximate.
While approximate algorithms are often sufficient for achieving good
generalization performance, we are interested in developing exact algorithms in
this case, and study whether they can achieve better reward learning than approximate algorithms.


Motivated by the above questions, we revisit the MaxEnt IRL model and present
some new perspectives, algorithms, and empirical insights.

\begin{enumerate}
    \item We provide a simple and rigorous derivation of the original MaxEnt IRL
			model using the Kullback-Liebler divergence, without using any heuristic
			argument.
			Our derivation provides a unified view for MaxEnt IRL and relative
			entropy IRL \cite{Boularias2011}, and highlights a key difference between
			these two frameworks.
			In addition, the connection between them suggests an model-free
			importance sampling algorithm for learning a MaxEnt IRL model
			(\cref{sec:the-principle-of-maximum-entropy}).
    \item We present an efficient exact forward-backward inference algorithm.
			This allows exact computation of the gradients used in reward learning.
			Unlike previous work, our algorithm does not assume that the demonstrations are of the same
			length. 
			While this increases the time complexity from linear in $L$ to quadratic
			in $L$, we bring the time complexity back to linear in $L$ using a
			padding trick (\cref{sec:a-new-algorithm}).
    \item We illustrate that published MaxEnt IRL algorithms
			\cite{Ziebart2008,Ziebart2010} are approximate, and 
			empirically show that exact algorithms improve reward learning.
			In addition, we show that our algorithm can scale up to a large,
			real-world dataset involving driver behaviour forecasting
			(\cref{sec:experimental-results}).
    \item We provide an open-source optimized reference implementation of our algorithms to allow easy application to other problems.
\end{enumerate}

Finally, we conclude with a discussion on opportunities for future work (\cref{sec:discussion}).

\section{Related Work}
\label{sec:related-work}

MaxEnt IRL theory was proposed in \parencite{Ziebart2008,Ziebart2010b}.
This framework has seen great success, and many extensions of this model have been proposed e.g. for continuous state and action spaces \parencite{Aghasadeghi2011,Kalakrishnan2013} or using deep non-linear reward representations \parencite{Wulfmeier2017,Finn2016}.
As pointed out in the introduction, we present a simple derivation of MaxEnt IRL
that improves the original heuristic derivation, leading to a unified view of
the original theoretical framework \parencite{Ziebart2008} and the Relative Entropy
(RE) IRL method \parencite{Boularias2011}.
In addition, the unified view directly leads to a model-free learning algorithm
for MaxEnt IRL, which is an analogue of that for RE-IRL.

The basic version of our algorithm is adapted from the exact inference
algorithms for linear-chain conditional random fields \parencite{Lafferty2001}.
We present our algorithm in a form that supports any combination of state-,
state-action, and/or state-action-state features.
In addition, our algorithm has several important features that make it flexible:
(a) It does not require the demonstration trajectories to be of the same
length, and is thus capable of handling variable length demonstrations, and 
(b) Our algorithm also supports discounting, episodic or continuing MDPs, and MDPs with state-dependent action sets, and a straight-forward application of the methods from \parencite{Wulfmeier2017} can extend
our algorithm to the case of learned non-linear reward functions, such as a deep
neural networks.
Importantly, while adding support for variable-length paths increases the time complexity of our basic algorithm, we are able to achieve linear time complexity in the length of the longest demonstration path $L$ using a padding trick.

Our algorithms can be viewed as special cases of the well-known sum-product
algorithm for graphical models \parencite{kschischang2001factor}.
This perspective allows extending our algorithm to handle more complex features.
We leave this to future work.


\section{Inverse Reinforcement Learning}
\label{sec:irl}

We consider IRL in the context of an MDP $\MDP = \{\States, \Actions, p_0, T, \gamma, R\}$, with discrete states $s \in \States$, discrete actions $a \in \Actions$, starting state distribution $p_0(s)$, transition dynamics $T = p(s' \mid s, a)$, a discount factor $\gamma \in [0, 1)$, and a reward function denoted $R$, which we define in further detail below.

For episodic MDPs, we also designate the non-empty sub-set of MDP states that are terminal $\States^T \subseteq \States$.
\Ie{} encountering any terminal state $s^T \in \States^T$ grants the agent reward for encountering that state $R(s^T)$, but then immediately ends the episode of interaction with the MDP.
This has important implications for the process by which we assume the IRL dataset is generated -- as we show below.

A policy $\pi(a \mid s)$ provides a (possibly deterministic) mapping from states to actions and describes a strategy to navigate the MDP.
We denote a `sample' from a policy as a state-action trajectory ending with a state $\tau = ((s_1, a_1), \dots, (s_m, \texttt{None}))$, with length denoted by $|\tau|$.
For convenience, we also denote $\Parent(s') \triangleq \{(s, a) : T(s, a, s') > 0\}$ as the set of $(s, a)$ tuples that are valid \emph{parents} of the state $s'$ according to the MDP dynamics, and $\Children(s) \triangleq \{(a, s') : T(s, a, s') > 0\}$ as the set of $(a, s')$ tuples that are valid \emph{children} of the state $s$ according to the MDP dynamics.

In general, the domain of a reward function could be the set of states $R_s(s)$, or state-action pairs $R_{sa}(s, a)$, or state-action-state tuples $R_{sas'}(s, a, s')$.
In the interests of completeness and accuracy, the derivations in the following sections proceed with the most general reward structure possible -- \ie{} we allow for MDPs that include any combination of these reward function types.
Furthermore, we limit our focus to linear reward functions with known basis feature functions, \ie{}
\begin{align}
    R_s(s) &\triangleq \theta_{s}^\top \phi_s(s)
    \\
    R_{sa}(s, a) &\triangleq \theta_{sa}^\top \phi_{sa}(s, a)
    \\
    R_{sas'}(s, a, s') &\triangleq \theta_{sas'}^\top \phi_{sas'}(s, a, s')
\end{align}
\noindent where we will drop the subscripts for brevity when context provides the needed clarity.
\Eg{} to transform one of our algorithms below to the case of an MDP that contains only state-action rewards, the reader could simply substitute $\theta_s = 0$ and $\theta_{sas'} = 0$, then simplify all the equations that contain these terms.

We also define the application of a reward function to a trajectory (taking into account discounting) as follows;
\begin{align}
    R(\tau) &\triangleq
    \sum_{t=1}^{|\tau|} \gamma^{t-1} R_s(s_t) +
    \sum_{t=1}^{|\tau|-1} \gamma^{t-1} R_{sa}(s_t, a_t) + \gamma^{t-1} R_{sas'}(s_t, a_t, s_{t+1})
    \\
    &= \theta_s^\top \phi_s(\tau) + \theta_{sa}^\top \phi_{sa}(\tau) + \theta_{sas'}^\top \phi_{sas'}(\tau),
\end{align}
\noindent where we have defined
\begin{align}
    \phi_s(\tau) &\triangleq \sum_{t=1}^{|\tau|} \gamma^{t-1} \phi_s(s_t)
    \\
    \phi_{sa}(\tau) &\triangleq \sum_{t=1}^{|\tau|-1} \gamma^{t-1} \phi_{sa}(s_t, a_t)
    \\
    \phi_{sas'}(\tau) &\triangleq \sum_{t=1}^{|\tau|-1} \gamma^{t-1} \phi_{sas'}(s_t, a_t, s_{t+1})
\end{align}
\noindent for convenience.

In our IRL setting, we are provided with a set of demonstration trajectories $\Data = \{\tau_1, \dots, \tau_N\}$, and a partial MDP definition $\MDP \backslash R$ -- \ie{} we know the MDP dynamics, but not the reward function parameter(s) $\Theta = \{\theta_s, \theta_{sa}, \theta_{sas'}\}$.
The goal is to identify these parameter vectors such that the demonstration data appear `optimal' according to some criteria (e.g. maximizing cumulative $\gamma$-discounted rewards).

We make no assumption that the demonstrated paths $\Data$ are of equal lengths -- \eg{} this can occur naturally in stochastic episodic MDPs where we assume some exogenous process allows the agent to re-start episodes after encountering a terminal state.
On the other hand, in a continuing (non-episodic) MDP, the demonstration data should technically consist of one continuous trajectory of interaction data, however we allow that there may be some exogenous process by which the episode of interaction can be terminated at any point and re-started -- a common practice in Reinforcement Learning experiments.
Thus, in both the episodic and non-episodic cases, we must be prepared to handle data with trajectories of varying lengths.
As we discuss below, previous MaxEnt IRL algorithms only supported datasets where the trajectories are all the same length -- a key limitation that we address in this chapter.

Returning to the general problem statement of Inverse Reinforcement Learning, \citet{Abbeel2004} showed that for an MDP with linear rewards, to learn a policy $\pi$ with the same value as the demonstrator, it suffices to match \emph{feature expectations}, \ie{}, choose parameters $\Theta$ that induce a policy $\pi$ such that
\begin{align}
    \E_{\tau \sim \pi}[\phi(\tau)] &= \E_{\tau \sim \Data}[\phi(\tau)]
\end{align}
\noindent where the RHS are empirical expectations over the demonstration data.
While a useful starting point, this problem is ill-posed, because generally, many polices have matching feature expectations.
The problem is further complicated by the fact that positive-affine reward `shaping' transformations do not change the optimal policy \parencite{Ng1999}.

Early IRL methods generally relied on heuristics or probabilistic assumptions to resolve the ambiguity of a consistent reward.
On the other hand, the Maximum Entropy IRL approach provides a principled way to identify unique reward parameters.

\section{The Principle of Maximum Entropy}
\label{sec:the-principle-of-maximum-entropy}

The MaxEnt IRL model defines a distribution on the set $\Tau$ of all feasible
trajectories as
\begin{align}
    p_{\Theta}(\tau) \triangleq \frac{
        q(\tau)e^{R(\tau)}
    }{ 
        Z(\Theta)
    },
    \label{eq:maxent-dist}
\end{align}
where $q$ is the (un-normalized) distribution induced by MDP dynamics alone
\begin{align}
    q(\tau) \triangleq p_0(s_1) \prod_{t=1}^{|\tau|-1} T(s_{t+1} \mid s_t, a_t),
\end{align}
and 
\begin{align}
    Z(\Theta) \triangleq \sum_{\tau' \in \Tau} q(\tau') e^{R(\tau)},
\end{align}
is the normalizing constant often known as the \textit{partition function}.
In addition, the parameters $\Theta$ are chosen to maximize the
log-likelihood given $\Data$,
\begin{align}
    \label{eq:log-likelihood}
    \ell(\Theta) =
    \theta^\top \E_{\tau \sim \Data}[\phi(\tau)] + 
    \frac{1}{|\Data|}\sum_{\tau \in \Data} \log q(\tau)
    - \log Z(\Theta),
\end{align}
We can interpret $p_\Theta(\tau)$ as a non-stationary policy, which is more expressive than a stationary policy as it can vary over time-steps.
Furthermore, the MaxEnt IRL framework allows all possible behaviors to be jointly learned due to global normalization.
This makes it potentially more powerful in complex domains, as compared to models which learn a stationary policy or do not perform global normalization.   

As a contribution of this paper, we show that the MaxEnt IRL model is the solution of
\begin{align}
    \min_p ~ \kld{p}{q}&  \label{eq:kl}
    \\
    \text{s.t.} \qquad
    \E_{\tau \sim p}[\phi(\tau)] &= \E_{\tau \sim \Data}[\phi(\tau)], 
    \nonumber
    \\
    \sum_{\tau \in \Tau} p(\tau) &= 1,
    \nonumber
    \\
    p(\tau) &\ge 0 & \forall \tau \in \Tau,
    \nonumber
\end{align}
\noindent where
\begin{align}
    \kld{p}{q} = \sum_{\tau \in \Tau} p(\tau) ( \log p(\tau) - \log q(\tau) ),
\end{align}
is the Kullback-Leibler divergence from $p$ to $q$.
The feature moment matching constraints ensure the learned non-stationary policy $p$ will match the preferences demonstrated in the data, while the minimization objective forces the model close to the natural dynamics of the MDP.
This leads to a unique reward parameter solution, thus resolving the reward ambiguity problem.
Our proof is straight-forward, and similar in nature to that for RE-IRL \parencite{Boularias2011}.

\begin{thm}
    The solution of \cref{eq:kl} is given by \cref{eq:maxent-dist}. 
\end{thm}

\begin{proof}
    The Lagrangian of \cref{eq:kl} is given by,
    \begin{align}
        \mathcal{L}(
        p(\tau_1), \dots, &p(\tau_{|\Tau|}), \lambda_1, \dots, \lambda_{|\Tau|+4}, s_1, \dots, s_{|\Tau|}
        ) =
        \nonumber
        \\
        &\sum_{\tau \in \Tau} p(\tau) ( \log p(\tau) - \log q(\tau) )
        \nonumber
        \\
        & - \lambda_1^\top \left(
        \sum_{\tau \in \Tau} p(\tau) \phi(\tau) - \frac{1}{|\Data|} \sum_{\tau \in \Data} \phi(\tau)
        \right)
        \nonumber
        \\
        & - \lambda_2 \left(
        \sum_{\tau \in \Tau} p(\tau) - 1
        \right)
        \nonumber
    \end{align}
    \noindent where $\lambda_1, \lambda_2$ are the Lagrange multipliers.
    Differentiating \wrt{} $p(\tau)$, we get
    \begin{align}
        \nabla_{\Theta} \mathcal{L} &=
        \log p(\tau) - \log q(\tau)
        - \lambda_1^\top \phi(\tau)
        - \lambda_2 + 1
    \end{align}
    Re-defining $\lambda_2' = \lambda_2 + 1$, equating the partial derivatives with zero, and re-arranging, we get
    \begin{align}
        p(\tau) &= q(\tau) \exp(
        \lambda_1^\top \phi(\tau)
        )
        \exp(\lambda_2')
        \\
        p(\tau) &= \frac{
            q(\tau) \exp(
            \theta^\top \phi(\tau)
            )
        }{
            \exp{-\lambda_2'}
        }
        \\
        p_\Theta(\tau) &= \frac{
            q(\tau) e^{R(\tau)}
        }{
            Z(\Theta)
        }
    \end{align}
\end{proof}

We highlight here a few new insights from this interpretation.
First, this interpretation shows that the MaxEnt IRL model chooses a model that
best agrees with the transition dynamics under the feature matching constraints.
Second, RE-IRL and MaxEnt IRL are the same, however they use 
different reference distributions $q$.
Specifically, the reference distribution used in RE-IRL is 
$q_{\pi_{0}}(\tau) 
= q(\tau) w_{\pi_{0}}(\tau)$, where $\pi_{0}$ is a baseline policy,
and $w_{\pi_{0}} = \prod_{t=1}^{|\tau|-1} \pi_{0}(a_{t} \mid s_t)$.
While MaxEnt IRL and RE-IRL can now be seen as special cases of a unified
model, this also reveals a subtle but important difference:
RE-IRL aims to agree with a baseline policy and the transition dynamics, while
MaxEnt IRL aims to agree with the transition dynamics only.
If we only consider trajectories of the same length, then MaxEnt IRL can be seen
as RE-IRL with a uniform baseline policy.
Third, such a connection between MaxEnt IRL and RE-IRL allows us to directly
adapt the model-free importance sampling learning algorithm for RE-IRL to
MaxEnt IRL:
we simply replace all occurrences of $w_{\pi_{0}}(\tau)$ with the value `1' in the their
gradient estimator (\cf{} Eq. (8) in \parencite{Boularias2011}).
However, this approach may be biased towards shorter demonstrations as
they will have larger weights in the estimator.
we explore this approach further in \cref{sec:model-free-maxent}


\section{A New Algorithm}
\label{sec:a-new-algorithm}

To learn the MaxEnt IRL model in \cref{eq:maxent-dist}, we need to maximize the
log-likelihood, which is convex in $\Theta$ and thus can be maximized using
standard gradient-based methods --- in our experiments we used L-BFGS-B.
The value and the gradient of the log-likelihood, required in the optimization
algorithm, can be computed using the partition function $Z(\Theta)$ and
the marginal distributions $p_{\Theta,t}(s)$, $p_{\Theta,t}(s, a)$, and $p_{\Theta,t}(s, a, s')$, which denote the probability that the $t$-th state / state-action / state-action-state are $s$, $(s,a)$, or $(s,a,s')$ respectively when $\tau$ is sampled from the MaxEnt distribution $p_\Theta(\tau)$.
Specifically, with the partition function, we can easily compute the log-likelihood using \cref{eq:maxent-dist}.
On the other hand, with the marginals terms, the required gradients are given by,
\begin{align}
    \label{eq:s-grad}
    \nabla_{\theta_s} \ell(\Theta) &=
    \E_{\tau \sim \Data}[\phi_s(\tau)] - \sum_{s \in \States} \phi_s(s)
    \sum_{t=1}^L p_{\Theta,t}(s),
    \\
    \label{eq:sa-grad}
    \nabla_{\theta_{sa}} \ell(\Theta) &=
    \E_{\tau \sim \Data}[\phi_{sa}(\tau)] -
    \sum_{s \in \States}\sum_{a \in \Actions} \phi_{sa}(s, a)
    \sum_{t=1}^{L-1} p_{\Theta,t}(s, a).,
    \\
    \label{eq:sas-grad}
    \nabla_{\theta_{sas'}} \ell(\Theta) &=
    \E_{\tau \sim \Data}[\phi_{sas'}(\tau)] -
    \sum_{s \in \States}\sum_{a \in \Actions}\sum_{s' \in \States} \phi_{sas'}(s, a, s')
    \sum_{t=1}^{L-1} p_{\Theta,t}(s, a, s').
\end{align} 

Explicit computation of the partition function and marginals grows exponentially in time for longer demonstration paths, however the Markov property of the MDP allows us to decompose the partition and marginal feature values recursively with an efficient forward-backward algorithm.

This was previously discussed in \parencite{Ziebart2008}, and in an updated version of that paper \parencite{Ziebart2010}, however their algorithm relied on a heuristic for the case of stochastic MDPs.
We find that this leads to approximate gradients (see proofs in \cref{App:MaxEnt}) and negatively impacts the reward learning process (see experiments in \cref{sec:experimental-results}).
We also note that these algorithms were derived only for the case of un-discounted MDPs with state-based rewards 
--- our algorithm adds support for discounted MDPs, and for MDPs with reward functions consisting of any combination of state-, state-action, and/or state-action-state features.

Complementing and extending these previous works, we construct a novel dynamic program that computes MaxEnt IRL gradients that are exact, even for the case of stochastic MDPs.
The algorithm utilizes partial versions of the partition function, known as message-passing variables, which we describe below.

\subsection{An intuition for message-passing algorithms}

To illustrate the derivation of our dynamic program, it is helpful to consider an example MDP with four states and a single action \cref{fig:linear-mdp}.

\begin{figure*}[h]
    \centering
    \includegraphics[width=0.5\linewidth]{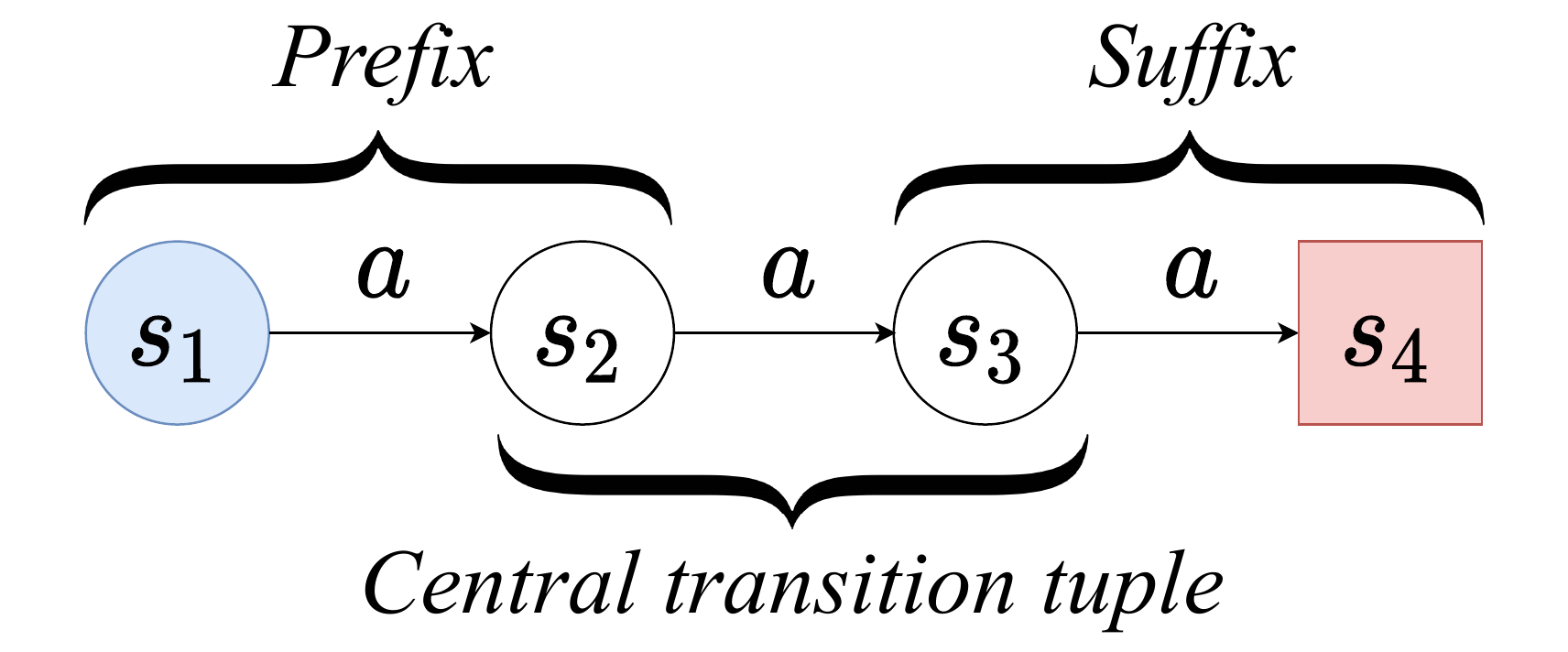}
    \caption{%
        An example MDP with four states and a single action.
        The MDP has a single deterministic starting state $s_1$ and a single terminal state $s_4$.
    }
    \label{fig:linear-mdp}
\end{figure*}

The set of paths of length $l = 4$ contains a single path,
\begin{align*}
    \tau = ((s_1, a), (s_2, a), (s_3, a), (s_4, \texttt{None})).
\end{align*}
Inspecting the partition contribution from this path, it is evident that we can decompose this term into three components: a `prefix' that ends with the state $s_2$, the central $(s_2, a, s_3)$ tuple, and a `suffix' that begins with the state $s_3$.
\begin{align*}
    p(\tau_1) &=
    \overbrace{
        \left(
        p_0(s_1) e^{R(s_1)}
        ~
        T(s_2 \mid s_1, a) e^{R(s_1, a) + R(s_1, a, s_2) + \gamma R(s_2)}
        \right)
    }^{\text{Length $t$ prefix ending at $s_2$}}
    \nonumber
    \\
    &\qquad \times \left(
    T(s_3 \mid s_2, a) e^{\gamma R(s_2, a) + \gamma R(s_2, a, s_3)}
    \right)
    \\
    \nonumber
    &\qquad\qquad \times \underbrace{
        \left(
        e^{\gamma^2 R(s_3)}
        ~
        T(s_4 \mid s_3, a) e^{\gamma^2 R(s_3, a) + \gamma^2 R(s_3, a, s_4) + \gamma^3 R(s_4)}
        \right)
    }_{\text{Length $l - t$ suffix beginning with $s_3$}}
\end{align*}
The same prefix-suffix pattern holds true when expressing the partition contribution for any path in a general MDP, however the prefix and suffix must sum over \emph{all} possible paths leading up to, or away from the central transition tuple.
Specifically, for a set of paths of lengths exactly $l$, the marginal state-action probability for a tuple $(s, a, s')$ occuring a time step $t$ will consist of three components:

\begin{enumerate}
    \item A path prefix counting the probability mass for all length $t$ paths that end at $s$
    \item The actual probability of the $(s, a, s')$ event
    \item A suffix counting the probability mass for all length $l - t$ paths beginning with $s'$.
\end{enumerate}

The prefix (suffix) term is known in the dynamic programming literature as the forward (backward) message-passing variable, as it functions to pass probability `messages' forward (backward) to (from) the central transition tuple.
Below, we show that for the Maximum Entropy behaviour model, the forward and backward message passing variables exhibit recursive sub-structure, which allows computing them efficiently with a dynamic program.

\subsection{Forward message passing variable}
We assume all trajectories in $\Tau$ have length at most $L$.
We define a forward message-passing variable that computes the partition contribution for paths of length $l$ that \emph{end} at a given state $s$,
\begin{align}
    \alpha_l(s) &\triangleq
    \sum_{\tau \in \Tau : |\tau|=l, \tau_l=s}
    q(\tau)
    e^{R(\tau)}
    & 1 \le l \le L.
\end{align}

Un-rolling this definition (i.e. fixing a base-case and re-writing the recurrence accordingly) leads to an expression for $\alpha_l(s)$,
\begin{align}
    \label{eq:alpha-basecase}
    \alpha_1(s) &= p_0(s) e^{R(s)}
    \\
    \label{eq:alpha-recurrence}
    \alpha_{l+1}(s') &=
    \sum_{(s, a) \in \Parent(s')}
    \alpha_l(s)
    T(s' \mid s, a)
    e^{
        \gamma^{l-1} R(s, a) +
        \gamma^{l-1} R(s, a, s') +
        \gamma^l R(s')
    }
    & 1 \le l < L.
\end{align}

\subsection{Backward message passing variable}
We define an analogous backward message-passing variable which counts the partition contribution for length $t$ suffixes within paths of total length $l$, where the path suffix \emph{starts} at a given state $s$,
\begin{align}
    \beta_{l,t}(s) &\triangleq
    \sum_{\tau\in\Tau:|\tau|=t,\tau_1=(s,\cdot)}
    q'(\tau)
    e^{R(\tau)}
    & 1 \le l \le L, 1 \le t \le l,
\end{align}

\noindent where $q'(\tau) \triangleq \prod_{t=1}^{|\tau|-1}T(s_{t+1}|s_t, a_t)$ is the same as $q(\tau)$, but does not include the starting state distribution.
Un-rolling the definition of $\beta_{l,t}(s)$ gives an analogous recurrence;
{\small
\begin{align}
    \label{eq:beta-basecase}
    \beta_{l,1}(s) &=
    e^{\gamma^{l-1} R(s)}
    \\
    & \qquad\qquad\qquad\qquad\qquad\qquad\qquad\qquad\qquad\qquad\qquad\qquad 1 \le l \le L
    \nonumber
    \\[10pt]
    \label{eq:beta-recurrence}
    \beta_{l,t+1}(s) &=
    \sum_{(a,s') \in \Children(s)}
    T(s' \mid s, a)
    e^{
        \gamma^{l-t-1} R(s)
        + \gamma^{l-t-1} R(s, a)
        + \gamma^{l-t-1} R(s, a, s')
    }
    \beta_{l,t}(s')
    \\
    \nonumber
    & \qquad\qquad\qquad\qquad\qquad\qquad\qquad\qquad\qquad\qquad\qquad\qquad 1 \le l \le L, 1 \le t < l.
\end{align}
}

\subsection{Partition and marginal calculations}
\label{subsec:partition-and-marginal-nopadding}
Now we describe the dynamic program to compute the terms of interest.
For a set of paths of lengths $1 \le l \le L$ the partition function value is given by summing the backward message passing values,
\begin{align}
    \label{eq:partition-variablelength}
    Z(\Theta) = \sum_{l=1}^L \sum_{s \in \States} \alpha_l(s),
\end{align}

\noindent and the marginal distributions are given by,
{\small
    \begin{align}
        p_{\Theta,t}(s) &=
        \frac{1}{Z(\Theta)}
        \alpha_t(s)
        \left(
        1
        + \sum_{(a, s') \in \Children(s)}
        T(s' \mid s, a)
        e^{
            \gamma^{t-1} R(s, a)
            + \gamma^{t-1} R(s, a, s') 
        }
        \sum_{l=t+1}^L
        \beta_{l,l-t}(s')
        \right)
        \nonumber
        \\
        & \qquad\qquad\qquad\qquad\qquad\qquad\qquad\qquad\qquad\qquad\qquad\qquad
        \forall t = 1, \dots, L,
        \label{eq:s-marginals}
        \\
        p_{\Theta,t}(s,a) &=
        \frac{1}{Z(\Theta)}
        \alpha_t(s)
        \sum_{s' \in \States}
        T(s' \mid s, a)
        e^{
            \gamma^{t-1} R(s, a)
            + \gamma^{t-1} R(s, a, s') 
        }
        \sum_{l=t+1}^L
        \beta_{l,l-t}(s')
        \nonumber
        \\
        & \qquad\qquad\qquad\qquad\qquad\qquad\qquad\qquad\qquad\qquad\qquad\qquad
        \forall t = 1, \dots, L-1,
        \label{eq:sa-marginals}
        \\
        p_{\Theta,t}(s,a,s') &=
        \frac{1}{Z(\Theta)}
        \alpha_t(s)
        T(s' \mid s, a)
        e^{
            \gamma^{t-1} R(s, a)
            + \gamma^{t-1} R(s, a, s') 
        }
        \sum_{l=t+1}^L
        \beta_{l,l-t}(s')
        \nonumber
        \\
        &
        \qquad\qquad\qquad\qquad\qquad\qquad\qquad\qquad\qquad\qquad\qquad\qquad
        \forall t = 1, \dots, L-1.
        \label{eq:sas-marginals}
    \end{align}
}
Before continuing, we briefly illustrate the function of these equations with a simple example.

\subsection{A Worked Example}

We return to the example MDP from \cref{fig:linear-mdp} to illustrate our algorithm.
We choose the set $\Tau$ with $1 \le l \le L = 4$, and illustrate our method with the following dataset,
\begin{align}
    \Data = \big\{\qquad&
    \nonumber
    \\
    & ((s_1, \texttt{None})), & (= \tau_1)
    \nonumber
    \\
    & ((s_1, a), (s_2, \texttt{None})), & (= \tau_2)
    \nonumber
    \\
    & ((s_1, a), (s_2, a), (s_3, \texttt{None})), & (= \tau_3)
    \nonumber
    \\
    & ((s_1, a), (s_2, a), (s_3, a), (s_4, \texttt{None})) & (= \tau_4)
    \nonumber
    \\
    \big\}\qquad\qquad&
    \label{eq:example-mdp-path-set}
\end{align}
The forward message passing matrix $\alpha_t(s)$ is given by \cref{eq:alpha-basecase,eq:alpha-recurrence},
\begin{table}[H]
    \centering
    $
    \alpha_l(s) = \begin{cases}
        ~\\
        ~\\
        ~\\
        ~
    \end{cases}
    $
    \hspace{-20pt}
    \resizebox{0.85\textwidth}{!}{
        \begin{tabular}{c|cccc}
            ~ & $l=1$ & $l=2$ & $l=3$ & $l=4$
            \\
            \midrule
            $s_1$ & $e^{R(s_1)}$ & $0$ & $0$ & $0$
            \\
            $s_2$ & $0$ & $\alpha_1(s_1)e^{R(s_1, a) + R(s_1, a, s_2) + \gamma R(s_2)}$ & $0$ & $0$
            \\
            $s_3$ & $0$ & $0$ & $\alpha_2(s_2)e^{\gamma R(s_2, a) + \gamma R(s_2, a, s_3) + \gamma^2 R(s_3)}$ & $0$
            \\
            $s_4$ & $0$ & $0$ & $0$ & $\alpha_3(s_3)e^{\gamma^2 R(s_3, a) + \gamma^2 R(s_3, a, s_4) + \gamma^3 R(s_4)}$
            \\
        \end{tabular}
    }
\end{table}
On the other hand, the backward message passing variable $\beta_{l,t}(s)$ must be computed for $1 \le l \le L = 4$, and for $1 \le t < l$.
These terms are given by \cref{eq:beta-basecase,eq:beta-recurrence}.
For length one paths ($l=1$),
\begin{table}[H]
    \centering
    $
    \beta_{1,t}(s) = \begin{cases}
        ~\\
        ~\\
        ~\\
        ~
    \end{cases}
    $
    \hspace{-20pt}
    \begin{tabular}{c|c}
        ~ & $t=1$
        \\
        \midrule
        $s_1$ & $e^{R(s_1)}$
        \\
        $s_2$ & $e^{R(s_2)}$
        \\
        $s_3$ & $e^{R(s_3)}$
        \\
        $s_4$ & $e^{R(s_4)}$
        \\
    \end{tabular}
\end{table}
For length two paths ($l=2$),
\begin{table}[H]
    \centering
    $
    \beta_{2,t}(s) = \begin{cases}
        ~\\
        ~\\
        ~\\
        ~
    \end{cases}
    $
    \hspace{-20pt}
    \begin{tabular}{c|cc}
        ~ & $t=2$ & $t=1$
        \\
        \midrule
        $s_1$ & $e^{R(s_1) + R(s_1,a) + R(s_1,a,s_2)}\beta_{2,1}(s_2)$ & $e^{\gamma R(s_1)}$
        \\
        $s_2$ & $e^{R(s_2) + R(s_2,a) + R(s_2,a,s_3)}\beta_{2,1}(s_3)$ & $e^{\gamma R(s_2)}$
        \\
        $s_3$ & $e^{R(s_3) + R(s_3,a) + R(s_3,a,s_4)}\beta_{2,1}(s_4)$ & $e^{\gamma R(s_3)}$
        \\
        $s_4$ & $0$ & $e^{\gamma R(s_4)}$
        \\
    \end{tabular}
\end{table}
For length three paths ($l=3$),
\begin{table}[H]
    \centering
    $
    \beta_{3,t}(s) = \begin{cases}
        ~\\
        ~\\
        ~\\
        ~
    \end{cases}
    $
    \hspace{-20pt}
    \begin{tabular}{c|ccc}
        ~ & $t=3$ & $t=2$ & $t=1$
        \\
        \midrule
        $s_1$ & $e^{R(s_1)+R(s_1,a)+R(s_1,a,s_2)}\beta_{3,2}(s_2)$ & $e^{\gamma R(s_1) + \gamma R(s_1,a) + \gamma R(s_1,a,s_2)}\beta_{3,1}(s_2)$ & $e^{\gamma^2 R(s_1)}$
        \\
        $s_2$ & $e^{R(s_2)+R(s_2,a)+R(s_2,a,s_3)}\beta_{3,2}(s_3)$ & $e^{\gamma R(s_2) + \gamma R(s_2,a) + \gamma R(s_2,a,s_3)}\beta_{3,1}(s_3)$ & $e^{\gamma^2 R(s_2)}$
        \\
        $s_3$ & $0$ & $e^{\gamma R(s_3) + \gamma R(s_3,a) + \gamma R(s_3,a,s_4)}\beta_{3,1}(s_4)$ & $e^{\gamma^2 R(s_3)}$
        \\
        $s_4$ & $0$ & $0$ & $e^{\gamma^2 R(s_4)}$
        \\
    \end{tabular}
\end{table}
And for length four paths ($l=4$),
\begin{table}[H]
    \centering
    $
    \beta_{4,t}(s) = \begin{cases}
        ~\\
        ~\\
        ~\\
        ~
    \end{cases}
    $
    \hspace{-20pt}
    \resizebox{0.85\textwidth}{!}{
        \begin{tabular}{c|cccc}
            ~ & $t=4$ & $t=3$ & $t=2$ & $t=1$
            \\
            \midrule
            $s_1$ & $e^{R(s_1) + R(s_1, a) + R(s_1, a, s_2) \beta_{4,3}(s_2)}$ & $e^{\gamma R(s_1) + \gamma R(s_1,a) + \gamma R(s_1,a,s_2)}\beta_{4,2}(s_2)$ & $e^{\gamma^2 R(s_1) + \gamma^2 R(s_1,a) + \gamma^2 R(s_1,a,s_2)}\beta_{4,1}(s_2)$ & $e^{\gamma^3 R(s_1)}$
            \\
            $s_2$ & $0$ & $e^{\gamma R(s_2) + \gamma R(s_2,a) + \gamma R(s_2,a,s_3)}\beta_{4,2}(s_3)$ & $e^{\gamma^2 R(s_2) + \gamma^2 R(s_2,a) + \gamma^2 R(s_2,a,s_3)}\beta_{4,1}(s_3)$ & $e^{\gamma^3 R(s_2)}$
            \\
            $s_3$ & $0$ & $0$ & $e^{\gamma^2 R(s_3) + \gamma^2 R(s_3,a) + \gamma^2 R(s_3,a,s_4)}\beta_{4,1}(s_4)$ & $e^{\gamma^3 R(s_3)}$
            \\
            $s_4$ & $0$ & $0$ & $0$ & $e^{\gamma^3 R(s_4)}$
            \\
        \end{tabular}
    }
\end{table}
We are now able to sum the forward message $\alpha_l(s)$ to compute the partition function value (\cref{eq:partition-variablelength}),
\begin{table}[H]
    \centering
    \resizebox{\textwidth}{!}{
    \begin{tabular}{rllll}
        $t=1$ & $t=2$ & $t=3$ & $t=4$
        \\
        \midrule
        $Z(\Theta) = e^{R(s_1)}$ & & &
        \\
        $+~ e^{R(s_1)}$ & $e^{R(s_1, a) + R(s_1, a, s_2) + \gamma R(s_2)}$ & &
        \\
        $+~ e^{R(s_1)}$ & $e^{R(s_1, a) + R(s_1, a, s_2) + \gamma R(s_2)}$ & $e^{\gamma R(s_2, a) + \gamma R(s_2, a, s_3) + \gamma^2 R(s_3)}$ &
        \\
        $+~ e^{R(s_1)}$ & $e^{R(s_1, a) + R(s_1, a, s_2) + \gamma R(s_2)}$ & $e^{\gamma R(s_2, a) + \gamma R(s_2, a, s_3) + \gamma^2 R(s_3)}$ & $e^{\gamma^2 R(s_3, a) + \gamma^2 R(s_3, a, s_4) + \gamma^3 R(s_4)}$
    \end{tabular}
    }
\end{table}
\vspace{-30pt}
\begin{align}
    ~
    \label{eq:example-partition}
\end{align}
Recalling that $\Tau = \Data$ in our example, and comparing with the set of paths (\cref{eq:example-mdp-path-set}), we can see that the partition value correctly accounts for the contributions of each of the four paths in $\Tau$.
We can then use the message passing variables to compute the marginal terms (\cref{eq:s-marginals,eq:sa-marginals,eq:sas-marginals}),
\begin{align*}
    p_{\Theta,1}(s_1) &= \frac{1}{Z(\Theta)} e^{R(s_1)} \left(
    1 + e^{R(s_1,a) + R(s_1,a,s_2)} \left(
    e^{\gamma R(s_2)} + ({\dots})e^{\gamma^2 R(s_3)} + ({\dots})e^{\gamma^3 R(s_4)}
    \right)
    \right)
    \\
    p_{\Theta,2}(s_2) &= \frac{1}{Z(\Theta)} ({\dots}) e^{\gamma R(s_2)} \left(
    1 + e^{\gamma R(s_2,a) + \gamma R(s_2,a,s_3)} \left(
    e^{\gamma^2 R(s_3)} + ({\dots})e^{\gamma^3 R(s_4)}
    \right)
    \right)
    \\
    p_{\Theta,3}(s_3) &= \frac{1}{Z(\Theta)} ({\dots}) e^{\gamma^2 R(s_3)} \left(
    1 + e^{\gamma^2 R(s_3,a) + \gamma^2 R(s_3,a,s_4)} \left(
    e^{\gamma^3 R(s_4)}
    \right)
    \right)
    \\
    p_{\Theta,4}(s_4) &= \frac{1}{Z(\Theta)} ({\dots}) e^{\gamma^3 R(s_4)} \left( 1 + (0) \right)
    \\[20pt]
    p_{\Theta,1}(s_1, a) &= \frac{1}{Z(\Theta)} e^{R(s_1)} e^{R(s_1,a) + R(s_1,a,s_2)} \left(
    e^{\gamma R(s_2)} + ({\dots}) e^{\gamma^2 R(s_3)} + ({\dots}) e^{\gamma^3 R(s_4)}
    \right)
    \\
    p_{\Theta,2}(s_2, a) &= \frac{1}{Z(\Theta)} ({\dots}) e^{\gamma R(s_2)} e^{\gamma R(s_2, a) + \gamma R(s_2, a, s_3)} \left(
    e^{\gamma^2 R(s_3)} + ({\dots}) e^{\gamma^3 R(s_4)}
    \right)
    \\
    p_{\Theta,3}(s_3, a) &= \frac{1}{Z(\Theta)} ({\dots}) e^{\gamma^2 R(s_3)} e^{\gamma^2 R(s_3, a) + \gamma^2 R(s_3, a, s_4)} \left(
    e^{\gamma^3 R(s_4)}
    \right)
    \\[20pt]
    p_{\Theta,1}(s_1, a, s_2) &= \frac{1}{Z(\Theta)} e^{R(s_1)} e^{R(s_1,a) + R(s_1,a,s_2)} \left(
    e^{\gamma R(s_2)} + ({\dots}) e^{\gamma^2 R(s_3)} + ({\dots}) e^{\gamma^3 R(s_4)}
    \right)
    \\
    p_{\Theta,2}(s_2, a, s_3) &= \frac{1}{Z(\Theta)} ({\dots}) e^{\gamma R(s_2)} e^{\gamma R(s_2, a) + \gamma R(s_2, a, s_3)} \left(
    e^{\gamma^2 R(s_3)} + ({\dots}) e^{\gamma^3 R(s_4)}
    \right)
    \\
    p_{\Theta,3}(s_3, a, s_4) &= \frac{1}{Z(\Theta)} ({\dots}) e^{\gamma^2 R(s_3)} e^{\gamma^2 R(s_3, a) + \gamma^2 R(s_3, a, s_4)} \left(
    e^{\gamma^3 R(s_4)}
    \right)
\end{align*}
where all $t,s,a,s'$ combinations that are not shown are equal to $0$, and intermediate exponential values have been collapsed $({\dots})$ for brevity.
The attentive reader can compare the values computed above with the columns of the partition value calculated in \cref{eq:example-partition} to confirm that the marginals correctly count all path contributions for each $t,s,a,s'$ combination.

\subsection{Pseudo-code listing}

We list the full pseudo-code for this algorithm in \cref{alg:maxent-irl}.
This dynamic program will compute the partition and marginal state values exactly, however requires storing $\beta_{l,t}$ for $1 \le t < l$, $1, \le l \le L$, and has polynomial time complexity $\Oh(|\States|^2 |\Actions| L^2)$, where $L$ is the length of the longest path in the demonstration dataset.
We show in the following section how a padding trick can be used to compute the same results linearly in $L$, and with less space required for the backward message variable $\bm{\beta}$.

\clearpage
\begin{algorithm}[H]
    \caption{ExactMaxEntIRLPoly --- Exact Maximum Entropy Inverse Reinforcement Learning that requires polynomial time and space complexity in the length of the longest demonstration path}
    \label{alg:maxent-irl}
    
    \small
    \DontPrintSemicolon
    
    \SetKwInput{Input}{Input}
    \SetKwInput{Output}{Output}
    \SetKw{Continue}{continue}
    
    \BlankLine
    \Input{%
        $\MDP \backslash R$, $\Data$, $\phi_s, \phi_{sa}, \phi_{sas'}$
    }
    \Output{%
        $\Theta = \{\theta_s, \theta_{sa}, \theta_{sas'}\}$, $Z(\Theta)$
    }
    
    \BlankLine
    $\overline{\phi_s} =
    \frac{1}{|\Data|}
    \sum_{\tau \in \Data}
    \sum_{t=1}^{|\tau|}
    \gamma^{t-1}
    \phi_s(s_t)$
    \tcc*[r]{Compute expert feature expectations}
    $\overline{\phi_{sa}} =
    \frac{1}{|\Data|}
    \sum_{\tau \in \Data}
    \sum_{t=1}^{|\tau|-1}
    \gamma^{t-1}
    \phi_{sa}(s_t, a_t)$
    \;
    $\overline{\phi_{sas'}} =
    \frac{1}{|\Data|}
    \sum_{\tau \in \Data}
    \sum_{t=1}^{|\tau|-1}
    \gamma^{t-1}
    \phi_{sas'}(s_t, a_t, s_{t+1})$ \;
    
    \BlankLine
    $L = \max_{\tau \in \Data} |\tau|$
    \tcc*[r]{Measure longest demonstration path}
    
    \BlankLine
    $\theta_s = \theta_{sa} = \theta_{sas'} = 0$
    \tcc*[r]{Begin gradient ascent Loop}
    \While{not converged}{
        \BlankLine
        $\alpha_1(s) = p_0(s) \exp(R(s))$
        \tcc*[r]{Forward message pass}
        \For{$t \gets 1 \dots L-1$}{
            $
            \alpha_{t+1}(s') = \sum_{(s, a) \in \Parent(s')} \alpha_t(s) T(s' \mid s, a)
            \exp (
            \gamma^{l-1} R(s, a) +
            \gamma^{l-1} R(s, a, s') +
            \gamma^l R(s')
            )
            $ \;
        }
        
        \BlankLine
        \For{$l \gets 1 \dots L$}{
            $\beta_{l,1}(s) = \exp(\gamma^{l-1} R(s))$
            \tcc*[r]{Backward message pass}
            \For{$t \gets 1 \dots l - 1$}{
                $
                \beta_{l,t+1}(s) = \sum_{(a,s') \in \Children(s)}
                T(s' \mid s, a)
                \exp(\gamma^{l-t-1}(R(s) + R(s,a) + R(s,a,s')))
                \beta_{l,t}(s')
                $ \;
            }
        }
        
        \BlankLine
        $Z(\Theta) = \sum_{l=1}^L \sum_{s \in \States} \alpha_l(s)$
        \tcc*[r]{Compute partition value}
        
        \BlankLine
        \BlankLine
        $p_{\Theta, L}(s) = \frac{1}{Z(\Theta)} \alpha_L(s)$
        \tcc*[r]{Compute marginal values}
        \For{$t \gets 1 \dots L-1$}{
            {\scriptsize
                $
                p_{\Theta, t}(s) = \frac{1}{Z(\Theta)}
                \alpha_t(s)
                \left(
                1 + 
                \sum_{(a, s') \in \Children(s)}
                T(s' \mid s, a)
                \exp(\gamma^{t-1} R(s, a) + \gamma^{t-1} R(s,a,s'))
                \sum_{l=t+1}^L
                \beta_{l,l-t}(s')
                \right)
                $
            }
            
            \BlankLine
            {\footnotesize
                $
                p_{\Theta,t}(s, a) = \frac{1}{Z(\Theta)}
                \alpha_t(s)
                \sum_{s' \in \States}
                T(s' \mid s, a)
                \exp(\gamma^{t-1} R(s, a) + \gamma^{t-1} R(s,a,s'))
                \sum_{l=t+1}^L
                \beta_{l,l-t}(s')
                $
            } \;
            
            \BlankLine
            $
            p_{\Theta,t}(s, a, s') = \frac{1}{Z(\Theta)}
            \alpha_t(s)
            T(s' \mid s, a)
            \exp(\gamma^{t-1} R(s, a) + \gamma^{t-1} R(s,a,s'))
            \sum_{l=t+1}^L
            \beta_{l,l-t}(s')
            $ \;
        }
        
        \BlankLine
        \BlankLine
        $\nabla_{\theta_s} =
        \overline{\phi_s} -
        \sum_{s \in \States}
        \phi_s(s)
        \sum_{t=1}^L
        p_{\Theta,t}(s)$
        \tcc*[r]{Compute gradients}
        
        $\nabla_{\theta_{sa}} =
        \overline{\phi_{sa}} -
        \sum_{s \in \States}
        \sum_{a \in \Actions}
        \phi_{sa}(s, a)
        \sum_{t=1}^{L-1}
        p_{\Theta,t}(s, a)$ \;
        
        $
        \nabla_{\theta_{sas'}} =
        \overline{\phi_{sas'}} -
        \sum_{s \in \States}
        \sum_{a \in \Actions}
        \sum_{s' \in \States}
        \phi_{sas'}(s, a, s')
        \sum_{t=1}^{L-1}
        p_{\Theta,t}(s, a, s')
        $ \;
        
        \BlankLine
        \BlankLine
        Update $\theta_s, \theta_{sa}, \theta_{sas'}$ using $\nabla_{\theta_s}, \nabla_{\theta_{sa}}, \nabla_{\theta_{sas'}}$ with chosen optimizer. \;
    }
    \Return{
        $\Theta = \{\theta_s, \theta_{sa}, \theta_{sas'}\}$ and $Z(\Theta)$
    }
\end{algorithm}

\section{A More Efficient Algorithm}

We now introduce a way to augment episodic and continuing MDPs (and their associated datasets of IRL demonstrations) so that all demonstrations will be of the same length, but the reward parameters learned using our Maximum Entropy IRL algorithm are unchanged.
This augmentation (a so-called `padding trick') has the effect of transforming an episodic MDP to a continuous MDP in a way which means that the demonstration trajectories from agents can all be `extended' until they all reach some upper length $L$.
This is done by adding a new state and action to the MDP, which form a recurrent sub-set of the state-action space of the MDP -- and by updating the transition dynamics and reward structure so that the corresponding Maximum Entropy probability distribution over trajectories is unchanged.
Using this approach we are able to transform the dataset of trajectories of varying lengths to a dataset of trajectories of a single fixed size -- which allows a reducing the computational complexity of our MaxEnt IRL dynamic program without changing the value of the calculated partition function or marginals.

\subsection{A padding trick for episodic and continuing MDPs}
\label{subsec:padding-trick}

Specifically, we augment the MDP by introducing an auxiliary state $s_a$ and action $a_a$.
To keep the derivation clear, we incorporate these elements into our existing notation as follows:
\begin{align}
    s_a &\notin \States^T, s_a \notin \States, a_a \notin \Actions
    \\
    \States^+ &\triangleq \States \cup \{s_a\}
    \\
    \Actions^+ &\triangleq \Actions \cup \{a_a\}
\end{align}
We illustrate the hierarchy of state and action sets in \cref{fig:state-classes,fig:action-classes}.

\begin{figure*}[t]
    \centering
    \begin{subfigure}[]{0.45\textwidth}
        \centering
        \includegraphics[width=\textwidth]{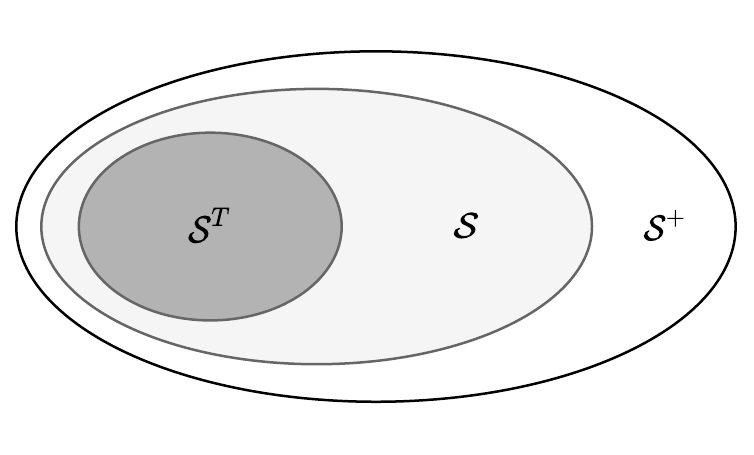}
        \caption{}
        \label{fig:state-classes}
    \end{subfigure}
    \hfill
    \begin{subfigure}[]{0.45\textwidth}
        \centering
        \includegraphics[width=\textwidth]{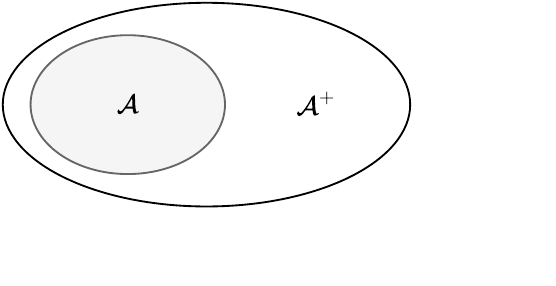}
        \caption{}
        \label{fig:action-classes}    
    \end{subfigure}
    \caption{%
        The hierarchy of state and action classes used in our notation for MDPs that have been augmented with our padding trick.
        \Nb{} $\States^+ - \States$ is defined as containing only a single element - the auxiliary state $s_a$.
        Likewise, $\Actions^+ - \Actions \triangleq \{a_a\}$.
    }
\end{figure*}

Our padding method requires that the auxiliary state and action satisfy the following properties \vav{} the dynamics of the augmented MDP;

\begin{enumerate}
    \item The agent may not start in the auxiliary state $s_a$
    \begin{align}
        p_0(s_a) \triangleq 0
    \end{align}
    \item The auxiliary state $s_a$ is absorbing
    \begin{align}
        p(s' \mid s_a, a) &\triangleq \begin{cases}
            1 & s' = s_a
            \\
            0 & \text{else}
        \end{cases}
        & \forall a \in \Actions^+
    \end{align}
    \item The auxiliary action always transitions deterministically to the auxiliary state
    \begin{align}
        p(s' \mid s, a_a) &\triangleq \begin{cases}
            1 & s' = s_a
            \\
            0 & \text{else}
        \end{cases}
        & \forall s \in \States^+
    \end{align}
    \item Terminal states transition to the auxiliary state no matter what action is taken
    \begin{align}
        p(s' \mid s, a) &\triangleq \begin{cases}
            1 & s' = s_a
            \\
            0 & \text{else}
        \end{cases}
        & \forall s \in \States^T, \forall a \in \Actions^+
    \end{align}
\end{enumerate}

These rules also imply the following updates to the \emph{Child set} and \emph{Parent set} operators.

\begin{itemize}
    \setcounter{enumi}{4}
    \item \emph{All} states (including terminal states and the auxiliary state) now feature the auxiliary action and state state in their children set
    \begin{align}
        \implies \Children(s) &\supseteq \{(a_a, s_a)\} & \forall s \in \States^+.
    \end{align}
    
    \item The auxiliary state contains all states (including itself) in it's parent set
    \begin{align}
        \implies \Parent(s_a) &\triangleq \{(s, a_a) : \forall s \in \States^+ \}.
    \end{align}

    \item The auxiliary state contains only the auxiliary action and state in it's child set
    \begin{align}
        \implies \Children(s_a) &\triangleq \{(a_a, s_a)\}.
    \end{align}

    \item Terminal states now have a child set spanning each of the set of all actions, followed by the auxiliary state
    \begin{align}
        \implies \Children(s) &\triangleq \{(a, s_a) : \forall a \in \Actions^+ \} & \forall s \in \States^T
    \end{align}
\end{itemize}

\Eg{} returning to the example of the linear MDP from \cref{fig:linear-mdp}, the updated MDP transition structure is as follows (note that $s_4$, which was formerly terminal, now has a successor state -- $s_a$);

\begin{figure*}[h]
    \centering
    \includegraphics[width=0.6\linewidth]{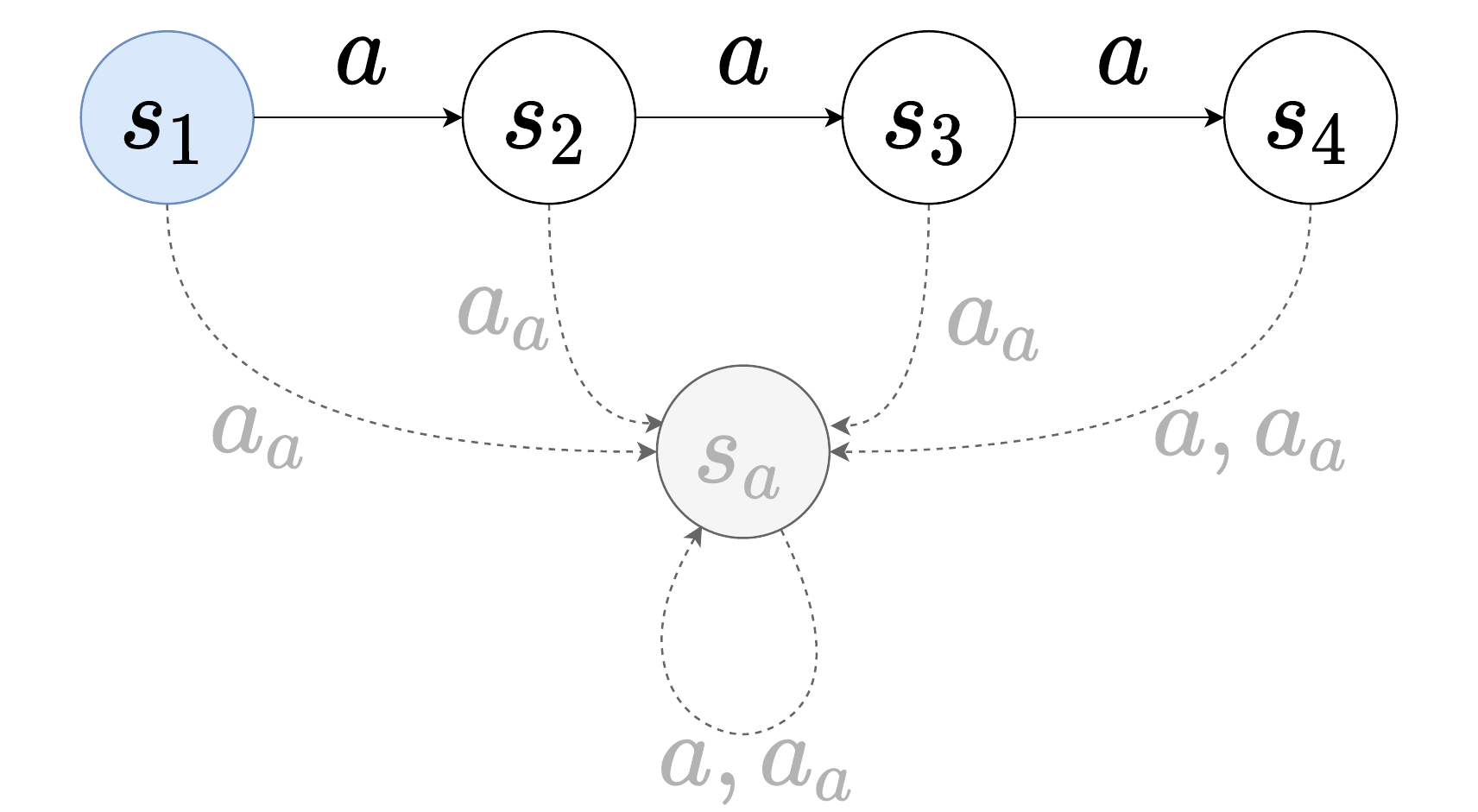}
    \caption{%
        The linear MDP from \cref{fig:linear-mdp}, after augmentation with the padding trick.
        Elements added as part of the padding trick are indicated in dashed lines and/or grey shading.
        Note that $s_4$ is no longer a terminal state.
    }
    \label{fig:linear-mdp-padded}
\end{figure*}

To complete the padding trick we must update the reward function in a way that will not modify the reward which is learned under the Maximum Entropy IRL model.
The requisite reward function changes are outlined in \cref{tab:padded-s-rewards,tab:padded-sa-rewards,tab:padded-sas-rewards}.
Essentially, these changes serve to infinitely discourage any actions that were impossible in the original MDP (e.g. executing actions in $\Actions$ after a terminal state), however allow the agent to transition to the auxiliary state at any point in time without incurring any modification to their gained reward.

\begin{table}[t]
    \renewcommand{\arraystretch}{1.3}
    \centering
    
    \begin{subtable}[l]{0.45\textwidth}
        \begin{center}
            \caption{$R(s)$}
            \label{tab:padded-s-rewards}
            \begin{tabular}{r|c}
                State & $R(s)$
                \\
                \midrule
                $s \in \States^T$ & 0
                \\
                $s \notin \States^T, s \ne s_a$ & $\theta_s^\top \phi(s)$
                \\
                $s = s_a$ & 0
            \end{tabular}
        \end{center}
    \end{subtable}
    \hfill
    \begin{subtable}[r]{0.45\textwidth}
        \begin{center}
            \caption{$R(s, a)$}
            \label{tab:padded-sa-rewards}
            \begin{tabular}{r|cccc}
                {} & \multicolumn{4}{c}{Action}
                \\
                \cline{3-4}
                State && $a \ne a_a$ & $a_a$ &
                \\
                \midrule
                $s \in \States^T$ && $-\infty$ & 0 &
                \\
                $s \notin \States^T, s \ne s_a$ && $\theta_{sa}^\top \phi(s, a)$ & 0 &
                \\
                $s = s_a$ && $-\infty$ & 0 &
            \end{tabular}
        \end{center}
    \end{subtable}
    \vspace{20pt}
    \begin{subtable}[h]{\textwidth}
        \begin{center}
            \caption{$R(s, a, s')$}
            \label{tab:padded-sas-rewards}
            \resizebox{\textwidth}{!}{%
                \begin{tabular}{rr|cccc|cccc}
                    {} & {} & \multicolumn{8}{c}{Action}
                    \\
                    \cline{4-9}
                    {} & {} && \multicolumn{3}{c|}{$a \ne a_a$} & \multicolumn{3}{c}{$a_a$} &
                    \\
                    \midrule
                    ~ & Second state && $s' \in \States^T$ & $s' \notin \States^T, s' \ne s_a$ & $s' = s_a$ & $s' \in \States^T$ & $s' \notin \States^T, s' \ne s_a$ & $s' = s_a$ &
                    \\
                    \midrule
                    \multirow{3}{*}{\rotatebox[origin=c]{90}{First state}}  & $s \in \States^T$ && $-\infty$ & $-\infty$ & $-\infty$ & $-\infty$ & $-\infty$ & 0 &
                    \\
                    & $s \notin \States^T, s \ne s_a$  && $\theta_{sas'}^\top \phi(s, a, s')$ & $\theta_{sas'}^\top \phi(s, a, s')$ & $-\infty$ & $-\infty$ & $-\infty$ & 0 &
                    \\
                    & $s = s_a$  && $-\infty$ & $-\infty$ & $-\infty$ & $-\infty$ & $-\infty$ & 0 &
                \end{tabular}
            }
        \end{center}
    \end{subtable}

    \caption{%
        Reward modifications for an MDP that has been augmented with the padding trick.
        (a) Modification for state-based reward function $R(s)$, (b) Modification for state-action based reward function $R(s, a)$, (c) Modification for state-action-state based reward function $R(s,a,s')$.
    }
    
\end{table}

Once the MDP definition has been updated to incorporate the auxiliary state and action, we can adjust the demonstration dataset to allow for a more efficient MaxEnt IRL algorithm, while still computing exact gradients.
The updates required for the demonstration trajectories are as follows.

For all sequences shorter than the longest demonstration path length $L = \max_{\tau \in \Data} |\tau|$, we pad them with auxiliary actions and states $((\cdot, a_a), (s_a, \cdot))$ until the sequence length is $L$.
For example, a $|\tau|=2$ sequence would be padded to length $L=4$ as follows;
\begin{align*}
    &((s_1, a_1), (s_2, \texttt{None}))
    \\
    \implies &((s_1, a_1), (s_2, \textcolor{lightgray}{a_a), (s_a, a_a), (s_a,} \texttt{None})).
\end{align*}
Once all demonstrations in the data $\Data$ are padded to the same length $L$, we can apply a simplified forward-backward algorithm to calculate the partition value exactly, but with better computational space and time complexity.
We describe this algorithm now.

\subsection{Padded message passing variables}

The forward message passing variable $\alpha_l(s)$ is computed as before, for all lengths $1 \le l < L$ and for states $s \in \States$ (\cref{eq:alpha-basecase,eq:alpha-recurrence}) -- \nb{} we do not need to bother computing $\alpha_l(s_a)$ because these terms are not needed in the partition and marginal calculations.

The backward message passing variable $\beta_{l,t}(s)$ still needs to be computed for suffix lengths $1 \le t < l$, however due to the padded sequences, we can fix $l = L$, removing one level of iteration.
We therefore drop the $l$ prefix and denote this term with the $t$ prefix only -- \ie{} $\beta_t(s)$, and compute this for states $s \in \States^+$. \Nb{} we \emph{will} need the terms $\beta_t(s_a)$, as they appear within the backward message recurrence and marginal calculations, however a simple inspection of \cref{tab:padded-s-rewards,tab:padded-sa-rewards,tab:padded-sas-rewards} shows that $\beta_t(s_a) = 1 ~ \forall t$.
The backward message recurrence is thus given by:
\begin{align}
    \label{eq:beta-basecase-padded}
    \beta_{1}(s) &= e^{\gamma^{L-1} R(s)}
    \\
    \label{eq:beta-recurrence-padded}
    \beta_{t+1}(s) &= \sum_{(a,s') \in \Children(s)}
    T(s' \mid s, a)
    e^{
        \gamma^{L-t-1} R(s)
        + \gamma^{L-t-1} R(s, a)
        + \gamma^{L-t-1} R(s, a, s')
    }
    \beta_{t}(s')
    & 1 \le t < L.
\end{align}

\subsection{Padded partition and marginal calculations}

With the padded MDP formulation, the partition function is unchanged (\cref{eq:partition-variablelength}), however we note that the inner summand is over the set $\States$, which does \emph{not} include the auxiliary state, but \emph{does} include states that were terminal states before the padding trick was applied.

The update for state marginal distributions no longer needs a summand over variable suffix lengths, thus reducing the time complexity.
That is,
\begin{align}
    \label{eq:s-marginals-padded}
    p_{\Theta,t}(s) &= \begin{cases}
        \frac{1}{Z(\Theta)}
        \alpha_t(s) 
        & t = L
        \\
        \frac{1}{Z(\Theta)}
        \alpha_t(s)
        \sum_{(a, s') \in \Children(s)}
        T(s' \mid s, a)
        e^{
            \gamma^{t-1} R(s, a)
            + \gamma^{t-1} R(s, a, s')
        }
        \beta_{L-t}(s')
        & 1 \le t < L,
    \end{cases}
    \\
    \label{eq:sa-marginals-padded}
    p_{\Theta,t}(s,a) &= 
    \frac{1}{Z(\Theta)}
    \alpha_t(s)
    \sum_{s' \in \States}
    T(s' \mid s, a)
    e^{
        \gamma^{t-1} R(s, a)
        + \gamma^{t-1} R(s, a, s')
    }
    \beta_{L-t}(s')
    \qquad\quad~~~ 1 \le t < L,
    \\
    \label{eq:sas-marginals-padded}
    p_{\Theta,t}(s,a,s') &= 
    \frac{1}{Z(\Theta)}
    \alpha_t(s)
    T(s' \mid s, a)
    e^{
        \gamma^{t-1} R(s, a)
        + \gamma^{t-1} R(s, a, s')
    }
    \beta_{L-t}(s')
    \qquad\qquad\qquad 1 \le t < L,
\end{align}

\noindent which must be computed for states $s \in \States, a \in \Actions, s' \in \States$ \ie{} everything but the auxiliary state and action.
We also draw the reader's attention to the fact that the summand over child tuples in \cref{eq:s-marginals-padded} \emph{does} include auxiliary action and state tuples, while the state summand in \cref{eq:sa-marginals-padded} does \emph{not} include auxiliary states.

We now briefly return to the example MDP to illustrate the consistency of the two dynamic programs.

\subsection{A Worked Example Revisited}

We return to the example MDP from \cref{fig:linear-mdp,fig:linear-mdp-padded} to demonstrate that the dynamic program with the padding trick faithfully computes the same values as the full dynamic program, while requiring less storage and time.

The path set $\Tau$ still consists of four paths, however the paths are now padded to be of equal length as follows;
\begin{align}
    \Tau = \big\{\qquad&
    \nonumber
    \\
    & ((s_1, \textcolor{lightgray}{a_a), (s_a, a_a), (s_a, a_a), (s_a,} \texttt{None})), & (= \tau_1)
    \nonumber
    \\
    & ((s_1, a), (s_2, \textcolor{lightgray}{a_a), (s_a, a_a), (s_a,} \texttt{None})), & (= \tau_2)
    \nonumber
    \\
    & ((s_1, a), (s_2, a), (s_3, \textcolor{lightgray}{a_a), (s_a,} \texttt{None})), & (= \tau_3)
    \nonumber
    \\
    & ((s_1, a), (s_2, a), (s_3, a), (s_4, \texttt{None})) & (= \tau_4)
    \nonumber
    \\
    \big\}\qquad\qquad&
    \label{eq:example-mdp-path-set-padded}
\end{align}

The recurrence for the forward message (\cref{eq:alpha-recurrence}) sums over parents of states $s \in \States$, which is exclusive of the auxiliary state $s_a$, therefore the computed forward message values are unchanged.

The backward message $\beta_t(s)$ is computed for $s \in \States^+$, as follows;
\begin{table}[H]
    \centering
    \hspace{-10pt}
    $
    \beta_t(s) = \begin{cases}
        ~\\
        ~\\
        ~\\
        ~
    \end{cases}
    $
    \hspace{-20pt}
    \resizebox{1.1\textwidth}{!}{
        \begin{tabular}{c|cccc}
            ~ & $t=4$ & $t=3$ & $t=2$ & $t=1$
            \\
            \midrule
            $s_1$ 
            & $e^{R(s_1) + R(s_1, a) + R(s_1, a, s_2) \beta_3(s_2)} + e^{R(s_1)}$
            & $e^{\gamma R(s_1) + \gamma R(s_1,a) + \gamma R(s_1,a,s_2)}\beta_2(s_2) + e^{\gamma R(s_1)}$
            & $e^{\gamma^2 R(s_1) + \gamma^2 R(s_1,a) + \gamma^2 R(s_1,a,s_2)}\beta_1(s_2) + e^{\gamma^2 R(s_1)}$
            & $e^{\gamma^3 R(s_1)}$
            \\
            $s_2$ 
            & $e^{R(s_2) + R(s_2, a) + R(s_2, a, s_3) \beta_3(s_3)} + e^{R(s_2)}$
            & $e^{\gamma R(s_2) + \gamma R(s_2,a) + \gamma R(s_2,a,s_3)}\beta_2(s_3) + e^{\gamma R(s_2)}$
            & $e^{\gamma^2 R(s_2) + \gamma^2 R(s_2,a) + \gamma^2 R(s_2,a,s_3)}\beta_1(s_3) + e^{\gamma^2 R(s_2)}$
            & $e^{\gamma^3 R(s_2)}$
            \\
            $s_3$ 
            & $e^{R(s_3) + R(s_3, a) + R(s_3, a, s_4)} + e^{R(s_3)}$
            & $e^{\gamma R(s_3) + \gamma R(s_3,a) + \gamma R(s_3,a,s_4)} + e^{\gamma R(s_3)}$
            & $e^{\gamma^2 R(s_3) + \gamma^2 R(s_3,a) + \gamma^2 R(s_3,a,s_4)}\beta_1(s_4) + e^{\gamma^2 R(s_3)}$ 
            & $e^{\gamma^3 R(s_3)}$
            \\
            $s_4$ & $1$ & $1$ & $1$ & $e^{\gamma^3 R(s_4)}$
            \\
            $s_a$ & $1$ & $1$ & $1$ & $1$
        \end{tabular}
    }
\end{table}

As the forward message is unchanged, the calculated value for the partition function will also be unchanged.

Finally, we compute the marginal terms for $s \in \States, a \in \Actions, s' \in \States$ as follows,
\begin{align*}
    p_{\Theta,1}(s_1) &= \frac{1}{Z(\Theta)} e^{R(s_1)} \left(
    ({\dots})e^{\gamma R(s_2)} \left(
    ({\dots}) e^{\gamma^2 R(s_3)} \left(
    ({\dots})e^{\gamma^3 R(s_4)}
    + 1
    \right)
    + 1
    \right)
    + 1
    \right)
    \\
    p_{\Theta,2}(s_2) &= \frac{1}{Z(\Theta)} ({\dots}) e^{\gamma R(s_2)} \left(
    e^{\gamma R(s_2, a) + \gamma R(s_2, a, s_3)} \left(
    ({\dots}) e^{\gamma^3 R(s_4)}
    + e^{\gamma^2 R(s_3)}
    \right)
    + 1
    \right)
    \\
    p_{\Theta,3}(s_3) &= \frac{1}{Z(\Theta)} ({\dots}) e^{\gamma^2 R(s_3)} \left(
    ({\dots}) e^{\gamma^3 R(s_4)} + 1
    \right)
    \\
    p_{\Theta,4}(s_4) &= \frac{1}{Z(\Theta)} ({\dots}) e^{\gamma^3 R(s_4)}
    \\[20pt]
    p_{\Theta,1}(s_1, a) &= \frac{1}{Z(\Theta)} e^{R(s_1)} \left(
    ({\dots}) e^{\gamma R(s_2)} \left(
    ({\dots}) e^{\gamma^2 R(s_3)} \left(
    ({\dots}) e^{\gamma^3 R(s_4)}
    + 1
    \right)
    + 1
    \right)
    \right)
    \\
    p_{\Theta,2}(s_2, a) &= \frac{1}{Z(\Theta)} ({\dots}) e^{\gamma R(s_2)} \left(
    ({\dots}) e^{\gamma^2 R(s_3)} \left(
    ({\dots}) e^{\gamma^3 R(s_4)}
    + 1
    \right)
    \right)
    \\
    p_{\Theta,3}(s_3, a) &= \frac{1}{Z(\Theta)} ({\dots}) e^{\gamma^2 R(s_3)} \left(
    ({\dots}) e^{\gamma^3 R(s_4)}
    \right)
    \\[20pt]
    p_{\Theta,1}(s_1, a, s_2) &= \frac{1}{Z(\Theta)} e^{R(s_1)} \left(
    ({\dots}) e^{\gamma R(s_2)} \left(
    ({\dots}) e^{\gamma^2 R(s_3)} \left(
    ({\dots}) e^{\gamma^3 R(s_4)}
    + 1
    \right)
    + 1
    \right)
    \right)
    \\
    p_{\Theta,2}(s_2, a, s_3) &= \frac{1}{Z(\Theta)} ({\dots}) e^{\gamma R(s_2)} \left(
    ({\dots}) e^{\gamma^2 R(s_3)} \left(
    ({\dots}) e^{\gamma^3 R(s_4)}
    + 1
    \right)
    \right)
    \\
    p_{\Theta,3}(s_3, a, s_4) &= \frac{1}{Z(\Theta)} ({\dots}) e^{\gamma^2 R(s_3)} \left(
    ({\dots}) e^{\gamma^3 R(s_4)}
    \right)
\end{align*}
\noindent where all $t,s,a,s'$ combinations that are not shown are equal to $0$, and once again the intermediate exponential values have been collapsed $({\dots})$ for brevity.
After re-arranging terms, the reader can verify that the computed marginal terms are indeed identical to those values calculated using the original dynamic program, thus concluding our demonstration.

\subsection{Pseudo-code listing}
\label{subsec:pseudocode-listing}

We list the full pseudo-code for the algorithm incorporating the padding trick in \cref{alg:maxent-irl-padded}.

\clearpage
\begin{algorithm}[H]
    \caption{ExactMaxEntIRL --- Exact Maximum Entropy Inverse Reinforcement Learning with Padded MDP Dynamics}
    \label{alg:maxent-irl-padded}
    
    \small
    \DontPrintSemicolon
    
    \SetKwInput{Input}{Input}
    \SetKwInput{Output}{Output}
    \SetKw{Continue}{continue}
    
    \BlankLine
    \Input{%
        $\MDP \backslash R$, $\Data$, $\phi_s, \phi_{sa}, \phi_{sas'}$
    }
    \Output{%
        $\Theta = \{\theta_s, \theta_{sa}, \theta_{sas'}\}$, $Z(\Theta)$
    }
    
    \BlankLine
    $\overline{\phi_s} =
    \frac{1}{|\Data|}
    \sum_{\tau \in \Data}
    \sum_{t=1}^{|\tau|}
    \gamma^{t-1}
    \phi_s(s_t)$
    \tcc*[r]{Compute expert feature expectations}
    $\overline{\phi_{sa}} =
    \frac{1}{|\Data|}
    \sum_{\tau \in \Data}
    \sum_{t=1}^{|\tau|-1}
    \gamma^{t-1}
    \phi_{sa}(s_t, a_t)$
    \;
    $\overline{\phi_{sas'}} =
    \frac{1}{|\Data|}
    \sum_{\tau \in \Data}
    \sum_{t=1}^{|\tau|-1}
    \gamma^{t-1}
    \phi_{sas'}(s_t, a_t, s_{t+1})$ \;
    
    \BlankLine
    $L = \max_{\tau \in \Data} |\tau|$
    \tcc*[r]{Measure longest demonstration path}
    
    \BlankLine
    Update $\MDP$ definition with auxiliary state $s_a$ and action $a_a$
    \tcc*[r]{Apply padding trick}
    Pad demonstration data $\Data$ with $((\cdot, a_a), (s_a, \cdot))$ sequences until all paths are of length $L$.
    
    \BlankLine
    $\theta_s = \theta_{sa} = \theta_{sas'} = 0$
    \tcc*[r]{Begin gradient ascent Loop}
    \While{not converged}{
        \BlankLine
        $\alpha_1(s) = p_0(s) \exp(R(s))$
        \tcc*[r]{Forward message pass}
        \For{$t \gets 1 \dots L-1$}{
            $
            \alpha_{t+1}(s') = \sum_{(s, a) \in \Parent(s')} \alpha_t(s) T(s' \mid s, a)
            \exp (
            \gamma^{l-1} R(s, a) +
            \gamma^{l-1} R(s, a, s') +
            \gamma^l R(s')
            )
            $ \;
        }
        
        \BlankLine
        $\beta_{1}(s) = \exp(\gamma^{L-1} R(s))$
        \tcc*[r]{Backward message pass}
        \For{$t \gets 1 \dots L-1$}{
            $
            \beta_{t+1}(s) = \sum_{(a,s') \in \Children(s)} T(s' \mid s, a)
            \exp(\gamma^{L-t-1}(R(s) + R(s,a) + R(s,a,s')))
            \beta_{t}(s')
            $ \;
        }
        
        \BlankLine
        $Z(\Theta) = \sum_{l=1}^L \sum_{s \in \States} \alpha_l(s)$
        \tcc*[r]{Compute partition value}
        
        \BlankLine
        \BlankLine
        $p_{\Theta, L}(s) = \frac{1}{Z(\Theta)} \alpha_L(s)$
        \tcc*[r]{Compute marginal values}
        \For{$t \gets 1 \dots L-1$}{
            $
            p_{\Theta, t}(s) = \frac{1}{Z(\Theta)}
            \alpha_t(s)
            \sum_{(a, s') \in \Children(s)}
            T(s' \mid s, a)
            \exp(\gamma^{t-1} R(s, a) + \gamma^{t-1} R(s,a,s'))
            \beta_{L-t}(s')
            $
            
            \BlankLine
            $
            p_{\Theta,t}(s, a) = \frac{1}{Z(\Theta)}
            \alpha_t(s)
            \sum_{s' \in \States}
            T(s' \mid s, a)
            \exp(\gamma^{t-1} R(s, a) + \gamma^{t-1} R(s,a,s'))
            \beta_{L-t}(s')
            $ \;
            
            \BlankLine
            $
            p_{\Theta,t}(s, a, s') = \frac{1}{Z(\Theta)}
            \alpha_t(s)
            T(s' \mid s, a)
            \exp(\gamma^{t-1} R(s, a) + \gamma^{t-1} R(s,a,s'))
            \beta_{L-t}(s')
            $ \;
        }
        
        \BlankLine
        \BlankLine
        $\nabla_{\theta_s} =
        \overline{\phi_s} -
        \sum_{s \in \States}
        \phi_s(s)
        \sum_{t=1}^L
        p_{\Theta,t}(s)$
        \tcc*[r]{Compute gradients}
        
        $\nabla_{\theta_{sa}} =
        \overline{\phi_{sa}} -
        \sum_{s \in \States}
        \sum_{a \in \Actions}
        \phi_{sa}(s, a)
        \sum_{t=1}^{L-1}
        p_{\Theta,t}(s, a)$ \;
        
        $
        \nabla_{\theta_{sas'}} =
        \overline{\phi_{sas'}} -
        \sum_{s \in \States}
        \sum_{a \in \Actions}
        \sum_{s' \in \States}
        \phi_{sas'}(s, a, s')
        \sum_{t=1}^{L-1}
        p_{\Theta,t}(s, a, s')
        $ \;
        
        \BlankLine
        \BlankLine
        Update $\theta_s, \theta_{sa}, \theta_{sas'}$ using $\nabla_{\theta_s}, \nabla_{\theta_{sa}}, \nabla_{\theta_{sas'}}$ with chosen optimizer. \;
    }
    \Return{
        $\Theta = \{\theta_s, \theta_{sa}, \theta_{sas'}\}$ and $Z(\Theta)$
    }
\end{algorithm}

This algorithm computes the same (exact) gradients as the basic algorithm described above in \cref{subsec:partition-and-marginal-nopadding}, however has linear time complexity in the size of the longest demonstration path $\Oh(|\States|^2 |\Actions| L)$.

\section{Implementation Comments}

A \naive{} implementation of \emph{any} MaxEnt algorithm may exhibit numerical floating-point overflow due to repeated exponentiation of rewards, especially for large positive reward values and/or long trajectories.
This can be avoided by using (natural) log-space variables and the standard `$\log$-$\zum$-$\exp$' transform when implementing the algorithm.
\Eg{}
\begin{align}
    y &= \log \sum_i \exp(x_i)
    \quad\Longleftrightarrow\quad
    y = c + \log \sum_i \exp(x_i - c),
    \\
    \text{where } c &= \max_{i} x_i.
    \nonumber
\end{align}
Finally, note that with appropriate modifications to the children and parent set operators $\Children(s)$ and $\Parent(s)$, our algorithms are also able to generalize to MDPs with state-dependent action sets $\Actions = \bigcup_{s \in \States} \Actions(s)$.

We provide an optimized reference implementation of this algorithm as a Python 3.6.9 package at our open-source code repository\footnote{\url{https://github.com/aaronsnoswell/unimodal-irl}}.
Our implementation utilizes the Numba Just-In-Time optimizing compiler \parencite{Lam2015} to achieve highly performant vectorized machine code for critical functions.

\section{Inference with the Maximum Entropy Behaviour Distribution}
\label{sec:me-inference}

After reward learning (\ie{} discovering the parameters for $p_\Theta(\tau)$), the maximum likelihood path between two states (or state-distributions) can be found using a Viterbi type dynamic program that has polynomial time complexity. 
If we denote the $i$-th state within a trajectory as $\tau^{(i)}$, and use $i = -1$ to denote the final state of a trajectory, this corresponds to solving the following optimization problem,
\begin{align}
    \argmax_{\tau \in \Tau} ~ &p_{\Theta}(\tau)
    \\
    s.t. \qquad
    p(\tau^{(1)} = s) &= f(s) & \forall s \in \States
    \\
    p(\tau^{(-1)} = s) &= g(s) & \forall s \in \States
\end{align}
\noindent for given (possibly degenerate) distributions $f$ and $g$.
In the special case when the learned weights are such that all $(s, a)$ choices incur a reward less than or equal to zero, any weighted shortest path search algorithm can be used (\eg{} Dijkstra's or Bellman-Ford), reducing the complexity for the problem of path inference conditioned on states.
We omit these algorithm for brevity, but refer the reader to our project repository.

On the other hand, we may wish use the learned maximum entropy path distribution to perform state inference, conditioned on partial paths.
\Eg{} \citet{Ziebart2008} show how Bayes' theorem can be applied to elegantly infer a distribution over destination states given an observation of the first few $(s, a)$ tuples in a trajectory.
If we extend our notation from above to use $\tau^{(A \to B)}$ to denote a path from state $s_A$ to state $s_B$, then we have the following useful result,
\begin{align}
    p(\tau^{(-1)} = s_G \mid \tau^{(A \to B)}) &\propto p(\tau^{(A \to B)} \mid \tau^{(-1)} = s_G) ~ p(s_G)
    \\
    &\propto \frac{
        \sum_{\tau^{(B \to G)}} q(\tau) \exp(R(\tau))
    }{
        \sum_{\tau'^{(A \to G)}} q(\tau') \exp(R(\tau'))
    }
    ~
    p(s_G)
\end{align}
\noindent where $p(s_G)$ is a prior distribution over destinations.
This can be used to rank possible destination states and/or to provide a distribution over expected path lengths -- all of which may be useful in  planning or navigation type problems.

These examples serve to illustrate the some of benefits of performing reward learning in the context of a distribution over behaviours, rather than an action-based distribution, as in some other IRL schemes.

\section{Experimental Results}
\label{sec:experimental-results}

We verify the function of our algorithm using several synthetic MDPs from the OpenAI Gym library \parencite{Brockman2016}, and demonstrate our algorithm's scalability with a large real-world problem in driver behaviour forecasting.

\subsection{Characterizing reward recovery performance}

First, we verify empirically that the reward function our algorithm learns becomes more accurate as the number of demonstration paths increases.

As a metric for IRL algorithm performance, we choose the Inverse Learning Error (ILE), first proposed in \parencite{Choi2011}.
The ILE measures the quality of a learned reward function $R_\text{L}$ by comparing it with the ground truth reward $R_\text{GT}$ 
--- however, \naive{} comparison of reward values is meaningless due to the reward ambiguity problems discussed in \cref{sec:irl}.
Instead, ILE compares \emph{value} of the ground truth optimal policy, with the \emph{value} of the optimal policy \wrt{} the learned reward.
The ILE is given by,
\begin{align}
    \text{ILE} \triangleq \norm{
        \bm{v}(\pi^*_{R_\text{GT}}) - \bm{v}(\pi^*_{R_\text{L}})
    }_1,
\end{align}

\noindent where $\bm{v}(\pi)$ indicates the vector of state-values \wrt{} the \emph{ground truth} reward $R_\text{GT}$ for any arbitrary policy $\pi$, and $\pi^*_{R_\text{GT}}$ and $\pi^*_{R_\text{L}}$ denote the optimal policy \wrt{} the ground truth and learned reward functions respectively.
Note that the ILE is on the range $[0, \infty)$, where lower values indicate a closer match to the ground truth reward, and the upper bound is specific to each MDP.

We evaluated the quality of our algorithm's learned rewards on three discrete state- and action space problems from the OpenAI Gym library (shown in \Cref{tab:experiment-quality-environments}).
For each environment, we find the optimal stationary deterministic policy using value iteration \parencite{Puterman1978}, then sample demonstration datasets containing an increasing number of paths.
For each dataset, we learn a reward function, then compute the corresponding ILE.
Each experiment is repeated 50 times to average over environment stochasticity, and we plot the ILE mean and 90\% confidence intervals over the 50 repeats.

\begin{table}[H]
    \centering
    \caption{Environments used for reward recovery experiment.}
    \label{tab:experiment-quality-environments}
    \scriptsize
    \renewcommand{\arraystretch}{1.2}
    \begin{tabular}{@{}r|lccl@{}}
        \toprule
        Environment & Dynamics & \thead{$|\States|$} & $|\Actions|$ & Reward Type
        \\
        \midrule
        Taxi-v3 & \makecell[l]{Stochastic starting state\\Deterministic transitions\\Episodic} & 500 & 6 & $R(s, a)$
        \\[10pt]
        FrozenLake4x4-v0 & \makecell[l]{Deterministic starting state\\Stochastic transitions\\Episodic} & 16 & 4 & $R(s)$
        \\[10pt]
        NChain-v0, N=10 & \makecell[l]{Deterministic starting state\\Stochastic transitions\\Continuing (non-episodic)} & 10 & 2 & $R(s, a)$
        \\
        \bottomrule
    \end{tabular}
\end{table}

The results are shown in \cref{fig:experiment-quality} 
--- our algorithm always converges to a lower ILE as the number of paths increases, indicating that we are able to recover accurate reward representations, and these reward functions are more accurate with increasing numbers of demonstration paths.

Also of interest is the fact that, for the \texttt{FrozenLake4x4} environment, our algorithm converges to a non-zero ILE.
We verified that this is because optimal policies in this MDP, which are used for sampling demonstrations, only solve the environment (reaching a goal state) in ${\sim}82\%$ of episodes.
If we artificially filter the optimal policy rollouts so that the demonstration data contain only successful episodes, our algorithm converges to $0.0 \pm 0.0$ ILE.

\begin{figure*}[t]
    \centering
    \includegraphics[width=\linewidth]{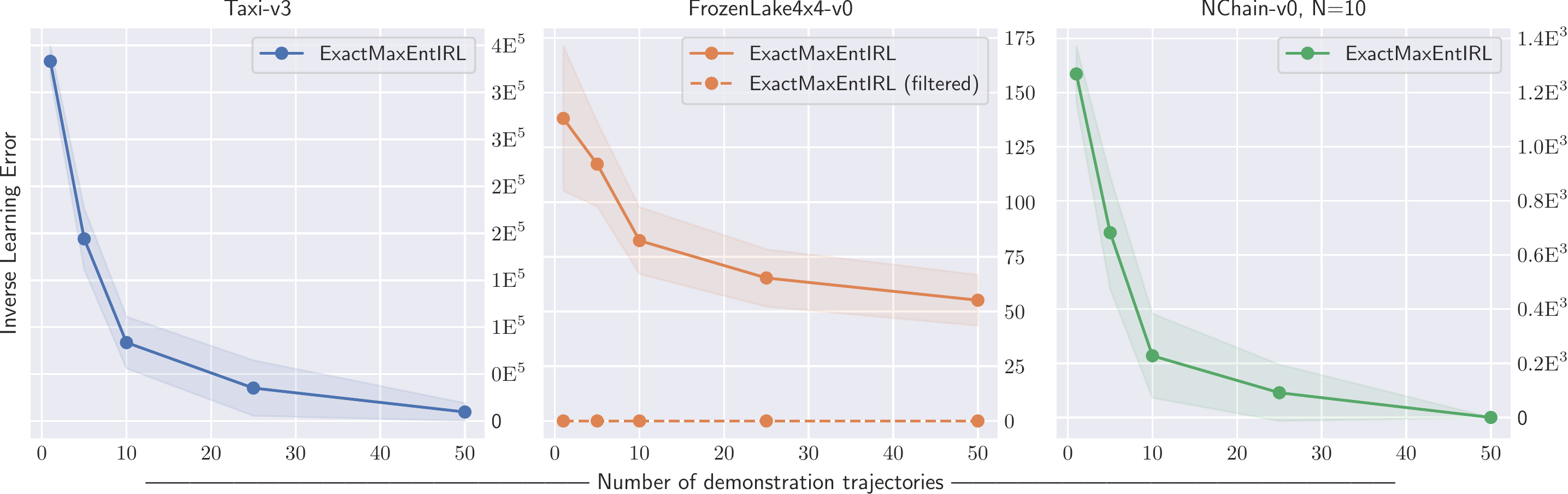}
    \caption{Our \textit{ExactMaxEntIRL} Algorithm Performance vs. Number of Demonstration Paths, showing means and 90\% confidence intervals over $50\times$ repeats.}
    \label{fig:experiment-quality}
\end{figure*}

\subsection{Empirical comparison with Ziebart's algorithms}

Without modification, the previous MaxEnt IRL algorithms by Ziebart \etal{} only support state-based reward features \parencite{Ziebart2008,Ziebart2010}.
The \texttt{FrozenLake4x4} environment consists of state-only rewards, which allows a fair comparison of the performance of our algorithm with these previous algorithms.

For the case of $N=50$ demonstration paths (not filtered to remove unsuccessful demonstrations), and with $50$ repeat experiments, our algorithm achieves an ILE mean and 90\% confidence interval of $55.2 \pm 18.3$, while Ziebart's 2008 algorithm achieves an ILE of $634.0 \pm 0.0$ and the 2010 algorithm achieves an ILE of $596.7 \pm 3.9$.

We also compute the log-likelihood of the demonstration data under each learned reward.
Our algorithm achieves a log-likelihood mean and 90\% confidence interval of $-133 \pm 7.78$ while Ziebart's 2008 and 2010 algorithms achieve log-likelihoods of $-336 \pm 31.2$ and $-365 \pm 33.0$ respectively.

For this specific environment, our algorithm out-performs Ziebart's Maximum Entropy algorithms on the ILE metric by a factor of over $10\times$, and the log-likelihood also confirms that our rewards are a better fit to the demonstrations.
These empirical data suggest that the approximate gradients from Ziebart's algorithms can sometimes have a negative effect on reward learning, which is also reflected in the results from our driver forecasting experiment, below.

\subsection{The padding trick improves computational efficiency}

Without the padding trick, our algorithm has a theoretical time complexity of $\Oh(|\States|^2 |\Actions| L^2)$, where $L$ is the length of the longest demonstration path.
With the padding trick, this dependence on $L$ becomes linear, $\Oh(|\States|^2 |\Actions| L)$.
We verify that this difference is important in practice, not just in theory.

To illustrate this, we again use the \texttt{FrozenLake} MDP template, but randomly generate unique environments of increasing size across three orders of magnitude.
For each problem size, we record the runtime required to learn a reward from a dataset of 10 paths using our algorithm in the padded, and non-padded configurations.
We repeat every experiment 30 times to average over variations in processor and memory utilization.
The experiments were performed on a Toshiba ThinkPad T480s laptop with an Intel {i7-8650U} {Quad-Core} CPU pinned at 2.1GHz, and with 24GB of RAM running Windows 10, 64-bit and using Python 3.6.9.

The results are shown in \cref{fig:experiment-speed}.
We plot the runtime mean and 90\% confidence interval vs. the problem size on a log-log scale.
The empirical behaviour aligns with our theoretical complexity analysis of the algorithm: the growth rate for both versions of our algorithm is slightly higher than linear in problem size $|\States|^2 |\Actions|$ 
--- a line with linear gradient is shown for comparison.
The results show small deviations from monotonic growth (e.g. the drop in runtime for the final point) 
--- we hypothesise that this is due to the low-level JIT compiler we utilize to optimize the Python code.

The results also confirm that the padding trick vastly improves the computational complexity of our algorithm, and that this improvement grows with the problem size.
For the small \texttt{FrozenLake4x4} MDP (problem size ${\sim}10^3$, third data-point from the left in figure), we see a ${\sim}10\times$ improvement in runtime, while for the larger \texttt{FrozenLake8x8} MDP (problem size ${\sim}10^{4.2}$, rightmost data-point in figure), we see a ${\sim}100\times$ improvement due to the padding trick.
These results are very encouraging, and suggest this algorithm is suitable for application to larger, real-world datasets, which we consider next.

\begin{figure}[t]
    \centering
    \includegraphics[width=\linewidth]{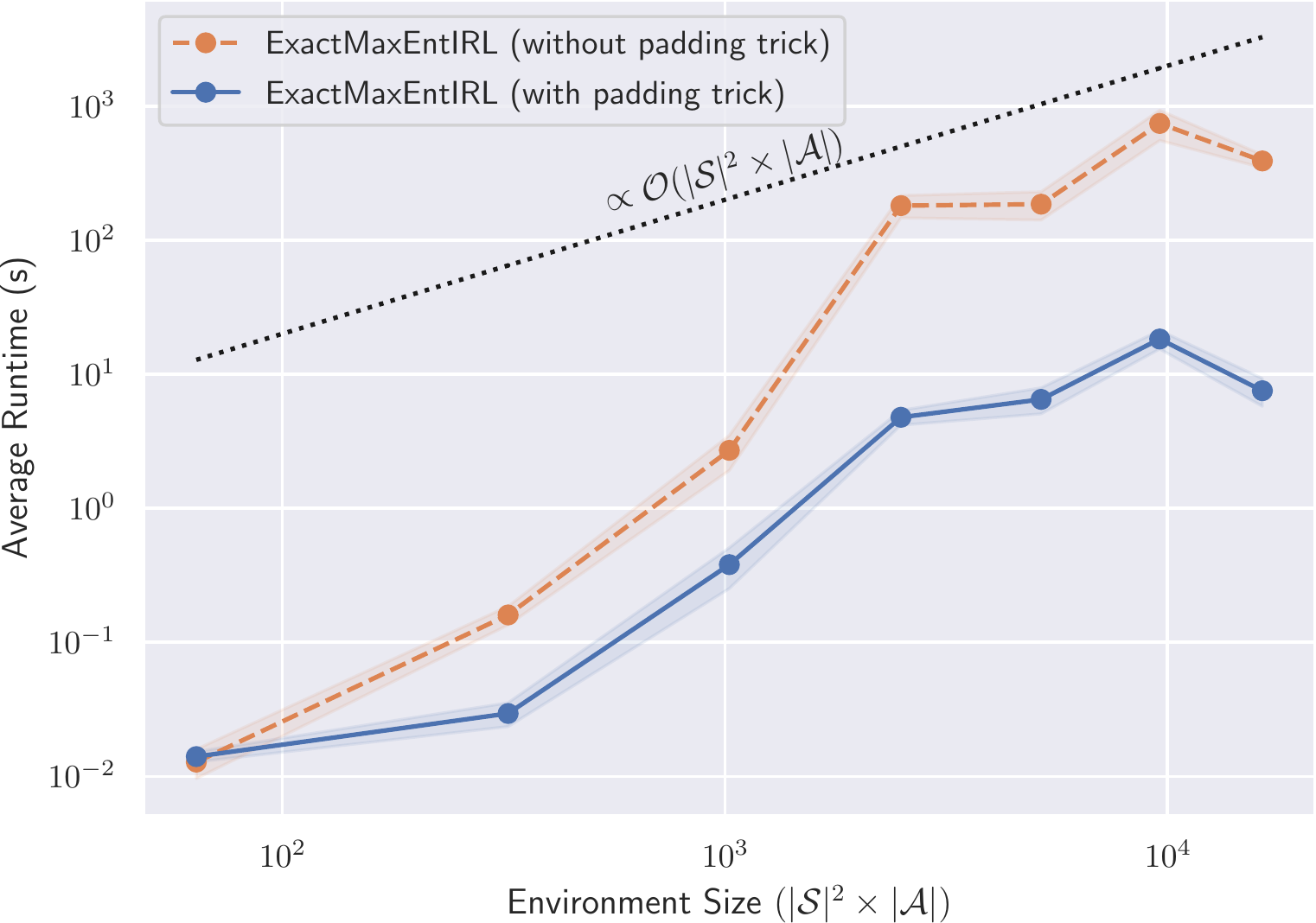}
    \caption{IRL Algorithm Runtime vs. Problem Size. Plots show mean and 90\% confidence intervals over $30\times$ repeats.}
    \label{fig:experiment-speed}
\end{figure}

\subsection{Example application: forecasting driver behaviour}

We demonstrate the utility of our algorithm by application to a large, real-world dataset similar to that used in the original MaxEnt IRL paper \parencite{Ziebart2008}.
The UCI Taxi Service Prediction dataset (the `Porto' dataset) contains over 1.7 million time-stamped GPS trajectories collected from the 442 taxis in the city of Porto, Portugal during 2013--14 \parencite{Dua2017}.

\subsubsection{Data pre-processing}
We adapted the Porto dataset to make it suitable for evaluating discrete state and action IRL algorithms.
Continuous GPS trajectories were fit using a particle filter to a discrete road network downloaded from \href{https://openstreetmap.org}{OpenStreetMap.org}.
Trajectories were removed from the dataset if they contained missing data, were shorter than 2 minutes in duration, contained cyclic paths, ventured outside a $15$km radius from the city, or if the particle filter did not converge.
This resulted in an MDP with 292,604 states (road segments), 594,933 actions (unique turns at road intersections), and with a stochastic starting state and deterministic transitions.
After filtering, the discretised path dataset contained 19,359 paths ranging in size from 5 to 840 states and length from 0.25 to 29km.
We excluded outlier paths with more than 400 states, and segmented into a 70\% training set (13551 paths) and 30\% held-out test set (5808 paths).

\subsubsection{Reward feature selection}
To allow comparison with Ziebart's algorithm, we selected a state-only reward feature representation.
As state features we utilised the number of lanes ($1$, $2$, or ${>}\,2$), road type (`local', `major', `highway', or `other'), speed limit (${<}\,35$km/h, ${35-55}$km/h, ${55-85}$km/h, ${>}\,85$km/h, or `unknown') and toll status (`toll' or `no toll'), giving a 14-dimensional indicator vector $\bm{I}(s)$ which we multiplied by the distance of a road segment in meters, $\bm{\phi_s}(s) \triangleq \bm{I}(s) \times \text{dist}(s)$.

\subsubsection{Evaluation Metrics}
The ILE metric used in our synthetic experiments requires that the ground truth reward function be known.
In the absence of a ground truth reward function, different evaluation metric(s) must be used.
We used two evaluation metrics, as follows:

\begin{itemize}
    \item \emph{Distance Match Percentage} $\in [0, 100]$, higher is better.
    Measures the percentage of distance of the predicted maximum likelihood path that matches the ground truth path.
    This is a domain-specific approximate measure of predictive accuracy of the policies induced by a learned reward.
    
    \item \emph{Feature Distance} $\in [0, \infty)$, lower is better, units are km.
    Measures the $L^2$ norm between the predicted maximum likelihood path's feature vector and that of the ground truth path.
    This is a domain-agnostic metric that quantifies how well the trained model matches the demonstrated preferences in the data.
\end{itemize}

\subsubsection{Results}
Using optimized implementations on a 24-core cluster workstation with Intel Xeon E5-2760 v3 CPUs at 2.4Ghz and 384GB of RAM, individual models took ${\sim}8$hrs to train to convergence using the L-BFGS-B optimizer, while evaluating a model against the test and the training data took ${\sim}60$hrs.

The results are shown in \cref{tab:porto-results}.
We compare the performance of our algorithm with that from \parencite{Ziebart2008}, as well as two baseline models - an agent that always chooses the shortest (distance) path\footnotemark{}, and a MaxEnt model with sampled random normal reward weights.
\footnotetext{%
    This is based on the assumption that taxi drivers (or their customers) might prefer a direct route to a destination.
}
For each model we report the distance match and feature distance metrics to three significant figures.
We report the median (and 90\% confidence interval of the median) as the result distributions are skewed 
--- however the non-overlapping confidence intervals and relative performance ranking for each metric are unchanged if the mean is used instead.

The results show that our model outperforms the others on predictive accuracy (the distance match metric), as well as in preference matching (the feature distance metric).
For both metrics, the algorithm from \parencite{Ziebart2008}, and the MaxEnt model with random normal weights perform significantly worse than either our algorithm or the shortest path heuristic 
--- by ${\sim}3$, and ${\sim}1$ orders-of-magnitude for the feature distance and distance match metrics respectively.

\begin{table*}[H]
    \centering
    \caption{Driver behaviour forecasting problem: experimental results}
    \label{tab:porto-results}
    \renewcommand{\arraystretch}{1.3}
    \begin{tabular}{@{}r|cc@{}}
        \toprule
        Algorithm & \makecell{Distance Match (\%)\\Median (90\% C.I.)} & \makecell{Feature Distance (km)\\Median (90\% C.I.)}
        \\
        \midrule
        ExactMaxEntIRL (Ours) & \textbf{64.3 (62.4 -- 66.1)} & \textbf{0.840 (0.797 -- 0.858)}
        \\
        Shortest Path & 50.5 (48.9 -- 51.7) & 0.995 (0.963 -- 1.02)
        \\
        \parencite{Ziebart2008} & 31.4 (30.4 -- 32.4) & 178 (177 -- 179)
        \\
        Random Weights & 27.1 (26.4 -- 28.0) & 220 (219 -- 221)
        \\
        \bottomrule
    \end{tabular}
\end{table*}

\section{Discussion}
\label{sec:discussion}

We presented new perspective and algorithms, including a new interpretation
that unifies MaxEnt IRL and RE-IRL with several implications, and an efficient
exact algorithm that leads to improved reward learning and is capable of scaling
up to a large real-world dataset.
We make an optimized implementation compatible with OpenAI Gym environments
publicly available to facilitate further research and applications.

We plan to follow up this work with some further developments. 
First, as mentioned in \cref{sec:related-work}, we can develop exact
algorithms to handle more complex features by adapting the sum-product
algorithm.
This can potentially lead to further performance improvement when complex
features are indeed necessary.
Second, we pointed out that our new interpretation of MaxEnt IRL suggests that
we can directly adapt the model-free importance sampling learning algorithm for
RE-IRL to MaxEnt IRL.
While this may be biased towards short demonstrations, this allows us to deal
with continuous MDPs.
In addition, in principle, we can choose an alternative reference distribution to
encode any other prior preference.
This needs to be further explored and empirically evaluated against the exact
algorithms.
Lastly, the MaxEnt IRL model in fact learns a reward function for a
non-stationary policy (that is, the MaxEnt trajectory distribution), however we (and others) treat the learned reward function as suitable for stationary policies, because it is
computationally easier to evaluate the performance of a stationary policy.
Our experiments suggest that the learned reward function are often suitable for a
stationary policy.
We hope to better understand when the reward function is suitable for a stationary
policy, and develop an effective method of using the learned reward together
with a non-stationary policy.

\section*{Acknowledgment}
\addcontentsline{toc}{section}{Acknowledgment}

Aaron Snoswell is supported by through an Australian Government Research Training Program Scholarship.

\clearpage
\appendix

\section{Appendix}

We briefly show that two previously published algorithms for the MaxEnt IRL problem do not always compute exact feature expectations.
We encourage the reader to reference the original algorithms \parencite{Ziebart2008,Ziebart2010} to follow the notation in the proofs.

\begin{proposition}
    Algorithm 1 from \parencite{Ziebart2008} computes incorrect feature expectations for any MDP with uniform starting state distribution and uniform dynamics, but non-uniform state reward function.
    \label{thm:zb08}
\end{proposition}

\begin{proof}
    Consider an MDP with uniform initial distribution and transition dynamics,
    \begin{align*}
        p_0(s) &\triangleq 1/|\States| &\forall s \in \States
        \\
        T(s' \mid s, a) &\triangleq 1 / |\States| &\forall s \in \States, a \in \Actions
    \end{align*}
    By `Algorithm 1', step (4), we have $D_{s,1} = p_0(s) = 1/|\States| ~\forall s \in \States$.
    Computing the next timestep for $D_{s,t}$ using step (5), we have
    \begin{align*}
        D_{s,2} &= \sum_{a \in \Actions} \sum_{s' \in \States} D_{s',1} ~ p(a \mid s) ~ T(s' \mid s, a)
        \\
        &= \sum_{a \in \Actions} \sum_{s' \in \States} \frac{1}{|\States|} ~ p(a \mid s) ~ \frac{1}{|\States|}
        \\
        &= \left(\sum_{a} p(a \mid s) \right) \left( \sum_{s'} \frac{1}{|\States|^2} \right)
        = (1) \left( \frac{1}{|\States|} \right)
        \\
        \implies D_{s,t} &= \frac{1}{|\States|} \qquad \forall s \in \States, t = 1, \dots, N.
    \end{align*}
    Thus, the computed state marginals do not depend on the reward function.
    This is only true for a degenerate reward $R(s) = \text{const.}$, for any non-uniform reward we will have a contradiction.
\end{proof}
The same paper was updated in 2010 with minor revisions to the algorithm regrading the handling of terminal states \parencite{Ziebart2010}.
The above proof also applies to this updated algorithm, as we show below.

\begin{proposition}
    Algorithm 1 from \parencite{Ziebart2010} computes incorrect feature expectations for any MDP with uniform starting state distribution and uniform dynamics, but non-uniform state reward function.
\end{proposition}

\begin{proof}
    Consider the same MDP as \cref{thm:zb08}, with uniform initial distribution and transition dynamics.
    \begin{align*}
        p_0(s) &\triangleq 1/|\States| &\forall s \in \States
        \\
        p(s' \mid s, a) &\triangleq 1 / |\States| &\forall s \in \States, a \in \Actions
    \end{align*}
    \noindent By `Algorithm 1', step (4), we have $D_{s,1} = p_0(s) = 1/|\States| ~\forall s \in \States$.
    Computing the next timestep for $D_{s',t}$ using step (5), we have
    \begin{align*}
        D_{s',2} &= \sum_{s \in \States} \sum_{a \in \Actions} D_{s,1} ~ p(a \mid s) ~ p(s' \mid s, a)
        \\
        &= \sum_{s \in \States} \sum_{a \in \Actions} \frac{1}{|\States|} ~ p(a \mid s) ~ \frac{1}{|\States|}
        = \frac{1}{|\States|^2} \sum_{s \in \States} \left( \sum_a p(a \mid s) \right)
        \\
        &= \frac{1}{|\States|^2} \sum_{s \in \States} (1)
        = \frac{1}{|\States|}
        \\
        \implies D_{s',t} &= \frac{1}{|\States|} \qquad \forall{s' \in \States, t = 1, \dots, N}
    \end{align*}
    
    Once again, for any non-uniform reward there will be a contradiction.
    
\end{proof}

\clearpage
\printbibliography[
    title={Bibliography}
]

\end{document}